\documentclass{article}

\PassOptionsToPackage{numbers, compress, sort}{natbib}

\usepackage[preprint]{neurips_2026}

\usepackage[utf8]{inputenc}
\usepackage[T1]{fontenc}
\usepackage[hypertexnames=false]{hyperref}
\usepackage{booktabs}
\usepackage{amsfonts}
\usepackage{nicefrac}
\usepackage{microtype}
\usepackage{xcolor}

\usepackage{amsmath,amssymb,amsthm}
\usepackage{mathtools}
\usepackage{graphicx}
\usepackage{xurl}
\usepackage{enumitem}
\usepackage{algorithm}
\usepackage{algpseudocode}
\usepackage{tcolorbox}
\usepackage{multirow}
\usepackage{colortbl}
\usepackage{wrapfig}
\usepackage{subcaption}
\usepackage{float}
\usepackage{titlesec}

\titlespacing*{\section}{0pt}{8pt}{6pt}
\titlespacing*{\subsection}{0pt}{3pt}{3pt}

\hypersetup{hidelinks}

\definecolor{tableheader}{RGB}{175,210,232}
\definecolor{tablelight}{RGB}{234,245,251}
\definecolor{tabledark}{RGB}{248,251,254}
\definecolor{tableruler}{RGB}{130,175,205}

\colorlet{mintheader}{tableheader}
\colorlet{mintstripe}{tablelight}
\colorlet{softgray}{tabledark}


\newtheorem{theorem}{Theorem}

\newtheorem{proposition}[theorem]{Proposition}

\newcommand{\best}[1]{\textbf{#1}}
\newcommand{\up}[1]{\textsuperscript{\(\uparrow#1\)}}
\newcommand{\downnote}{\textsuperscript{\(\downarrow\)}}

\title{Learning-Zone Energy: Online Data Selection for Efficient RL Post-Training}

\author{
\textsuperscript{1}Peng Cui\thanks{Equal contribution.},
~~\textsuperscript{2}Boyao Yang\footnotemark[1],
~~\textsuperscript{1}Jun Zhu\thanks{Corresponding author.} \\
\textsuperscript{1} Dept. of Comp. Sci. \& Tech., Institute for AI, BNRist Center, \\
Tsinghua-Bosch Joint ML Center, THBI Lab, Tsinghua University, Beijing 100084, China \\
\textsuperscript{2} Dept. of Automation, Tsinghua University, Beijing 100084, China \\
\texttt{xpeng.cui@gmail.com, boyaoyang16@gmail.com, dcszj@tsinghua.edu.cn}
}

\begin{document}
\addtocontents{toc}{\protect\setcounter{tocdepth}{-1}}

\maketitle

\begin{abstract}
Reinforcement Learning (RL) post-training has emerged as the dominant paradigm for eliciting mathematical reasoning in Large Language Models (LLMs), yet prevailing techniques such as GRPO and DAPO distribute rollout and gradient budgets nearly uniformly across prompts, squandering compute on samples that are already mastered or remain far beyond the model's current capability.
To address this fundamental inefficiency, we propose \textbf{Learning-Zone Energy}(\textsc{LZE}), a theoretically grounded, fully online data selection framework that concentrates computation on the model's \emph{active learning frontier}.
At its core, we define a closed-form \emph{Learning-Zone Energy Score} that fuses three complementary signals---an initial-difficulty anchor, a normalized outcome-uncertainty term, and a pass-rate momentum---into a single scalar that is provably aligned with the expected magnitude of group-relative policy gradient updates.
A forward pruner with replay further reduces wall-clock time cost by skipping rollout generation for persistently solved prompts while periodically checking for forgetting.
Evaluated on Qwen-family models (1.5B--8B) across GSM8K, MATH and DAPO-MATH, our method retains only 40\% of the training data per step yet matches or surpasses full-data baselines, with especially pronounced out-of-distribution gains on AIME25($+45.9\%$) and AMC23($+18.2\%$), alongside an estimated 36\% reduction in training FLOPs. Our code is available at \url{https://github.com/Stellaris167/LZE}.
\end{abstract}

\section{Introduction}

Reinforcement Learning (RL) with verifiable rewards has emerged as the dominant paradigm for improving mathematical reasoning in Large Language Models~\cite{brown2020language, openai2023gpt4}. In modern group-based RL pipelines such as GRPO~\cite{deepseek2024deepseekmath, deepseek2025deepseekr1} and DAPO~\cite{yu2025dapo}, the model generates multiple rollouts for each prompt, aggregates rewards within each group, and performs policy updates accordingly. However, these methods treat prompt groups nearly uniformly during training, even though the groups are highly heterogeneous in their usefulness: some prompts are already solved reliably and therefore contribute little additional learning signal, whereas others remain far beyond the model's current capability and produce almost no actionable gradient information. Consequently, a considerable fraction of the rollout and optimization budget is spent on prompt groups that are only weakly informative for policy improvement.

To mitigate this inefficiency, prior work has explored curriculum-style training, offline filtering, and online rejection-based heuristics. Early work on curriculum learning and self-paced learning showed that training efficiency can benefit from organizing examples by difficulty \cite{bengio2009curriculum,kumar2010self}. In the context of LLM post-training, offline filtering and self-training style methods can remove obviously uninformative examples before optimization begins \cite{singh2024beyond,gulcehre2023rest}.
However, such strategies are inherently static and cannot adapt as the policy evolves and the effective difficulty of a prompt changes over time. Online rejection methods are more adaptive, yet they typically apply binary keep-or-drop decisions once a prompt becomes all-correct or all-incorrect \cite{dong2023raft,wei2025minimalist}, making them insensitive to finer-grained differences in informativeness. A more recent online selection method~\cite{mao2026dynamicspredictive} avoids redundant rollouts by predicting prompt solvability via a hidden Markov model, but requires maintaining a learned generative state model, inevitably incurring additional computational overhead.
Therefore, existing approaches are often either static, overly rigid, or expensive. What is still missing is a lightweight online criterion that continuously adapts to the current policy, assigns prompt groups a graded notion of informativeness, and improves training efficiency without relying on heavy auxiliary machinery.

To address this gap, we propose \textbf{Learning-Zone Energy} (\textsc{LZE}), a simple and fully online data selection framework for RL post-training. Our key idea is to prioritise prompt groups in the model's current \textit{learning zone}: groups that are neither already mastered nor completely out of reach, but actively capable of producing meaningful policy updates. Concretely, we define an \emph{Energy Score} that combines three signals: an \emph{initial-difficulty anchor} that preserves attention on historically challenging prompts, a \emph{normalized uncertainty term} based on the current group pass rate that emphasises prompts near the decision boundary, and a \emph{momentum term} that highlights groups on which the policy is actively changing.
Beyond this practical intuition, we provide theoretical analysis for the design: in group-relative RL with Bernoulli rewards, the uncertainty term is aligned with the expected GRPO gradient variance under a standard fixed-baseline approximation (Theorem~\ref{thm:energy_gradient}), thereby explaining why frontier prompt groups are especially informative. We further show that the momentum term corresponds to the output of a
causal high-pass filter on the pass-rate sequence (Proposition~\ref{prop:ema_conv}), and that the score as a whole admits an interpretation as a sample-level attention mechanism. To further reduce training cost, we complement the backward scorer with a replay-based \emph{forward pruner} that skips rollout generation for prompts that have remained solved over consecutive epochs, with periodic replay to detect
forgetting.

Our contributions are threefold:
\begin{itemize}[leftmargin=*]
  \item \textbf{A principled online data selection framework.}
        The Learning-Zone Energy Score is a closed-form, per-step criterion that integrates difficulty history, outcome uncertainty, and pass-rate momentum into a single scalar provably aligned with the GRPO gradient variance. As it operates solely on group-level pass-rate statistics, \textsc{LZE} is compatible with any group-based RL post-training algorithm and requires no modification to the underlying optimizer.
        Retaining only 40\% of prompts per training step, \textsc{LZE} matches and even surpasses full-data baselines across all evaluated configurations, with particularly pronounced out-of-distribution gains: up to $+45.9\%$ on AIME\,25 and $+18.2\%$ on AMC\,23.

  \item \textbf{Theoretical analysis of gradient informativeness in RL post-training.}
        We prove that the Bernoulli variance $4p(1{-}p)$ governs the GRPO gradient variance up to a small-step-size approximation
        (Theorem~\ref{thm:energy_gradient}), establish an EMA--convolution duality linking the momentum term to causal high-pass filtering of the pass-rate sequence (Proposition~\ref{prop:ema_conv}).

  \item \textbf{Compute efficiency via dual-stage selection.}
        The backward selector and forward pruner together reduce theoretical training FLOPs by an estimated 36\%, consistent with the empirical FLOPs measurements across all model and dataset configurations reported in Table~\ref{tab:pass1_all}.
\end{itemize}

\begin{figure}[t]
\centering
\includegraphics[width=0.99\textwidth]{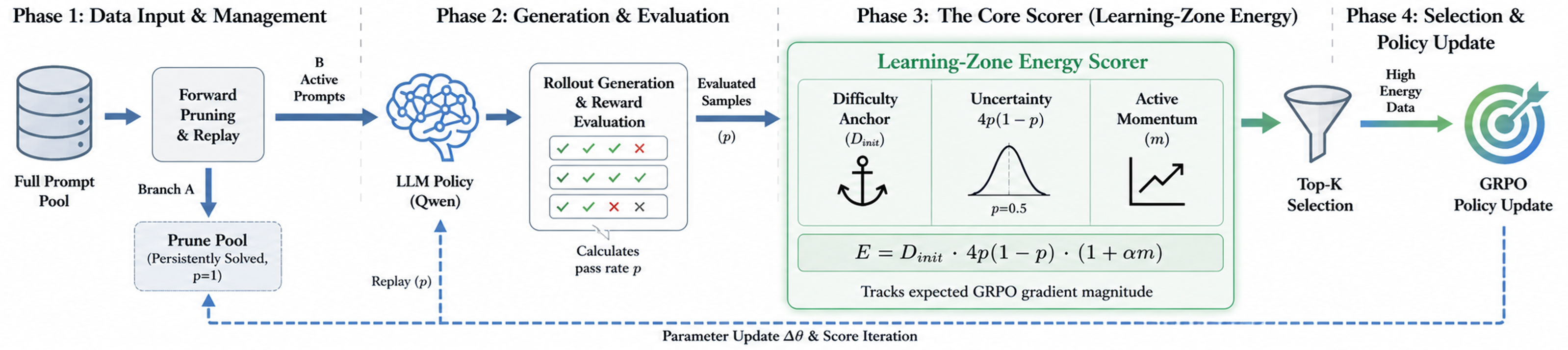}
\caption{\textbf{Overview of the proposed method.} The backward scorer computes the Learning-Zone Energy Score for each prompt group at every training step, guiding Top-$K$ selection for policy updates. The forward pruner tracks group pass rates at the epoch level and skips rollout generation for groups that have been stably solved, with replay providing a safety mechanism to detect forgetting.}
\label{fig:pipeline}
\vspace{-10pt}
\end{figure}
\section{Background}
\label{sec:background}
We review the group-based RL post-training paradigm, introduce the notation used throughout, and identify the gradient-informativeness property that motivates our approach.

\textbf{RL post-training for mathematical reasoning.}
Given a prompt dataset $\mathcal{D}=\{x_i\}_{i=1}^{N}$ and a policy $\pi_\theta$,
RL post-training maximises the expected reward:
\begin{equation}
\label{eq:rl_objective}
\max_\theta\;\mathbb{E}_{x\sim\mathcal{D},\;y\sim\pi_\theta(\cdot|x)}\bigl[r(x,y)\bigr],
\end{equation}
where $r(x,y)\in\{0,1\}$ is a rule-based binary verifier that checks the correctness
of response $y$ for prompt $x$.
This binary reward structure, standard in mathematical reasoning, makes the per-prompt \emph{pass rate} a natural proxy for the policy's current capability on that problem.

\textbf{Group Relative Policy Optimization (GRPO).}
GRPO~\cite{deepseek2024deepseekmath,deepseek2025deepseekr1} optimizes
Eq.~\eqref{eq:rl_objective} without a learned value network by normalizing rewards within each sampled group.
For prompt~$x_i$, the policy generates $n$ i.i.d.\ rollouts
$y_{i,1},\ldots,y_{i,n}\sim\pi_\theta(\cdot|x_i)$ with binary rewards
$r_{i,k}=r(x_i,y_{i,k})$.
The empirical pass rate $p_i^{(t)}=n^{-1}\sum_k r_{i,k}$ serves as the group
baseline, and the normalized group-relative advantage is
\begin{equation}
\label{eq:advantage}
\tilde{A}_{i,k} \;=\; r_{i,k} - p_i^{(t)}.
\end{equation}
GRPO then maximizes the clipped surrogate objective
\begin{equation}
\label{eq:grpo_obj}
\mathcal{L}_{\mathrm{GRPO}}(\theta)
= \frac{1}{N}\sum_{i=1}^{N}\frac{1}{n}\sum_{k=1}^{n}
  \min\!\left(
    \frac{\pi_\theta(y_{i,k}\mid x_i)}{\pi_{\theta_{\mathrm{old}}}(y_{i,k}\mid x_i)}
    \tilde{A}_{i,k},\;
    \mathrm{clip}\!\left(
      \frac{\pi_\theta(y_{i,k}\mid x_i)}{\pi_{\theta_{\mathrm{old}}}(y_{i,k}\mid x_i)},
      1{-}\epsilon, 1{+}\epsilon
    \right)\tilde{A}_{i,k}
  \right),
\end{equation}
where the clipped importance ratio prevents $\pi_\theta$ from deviating excessively
from the behaviour policy $\pi_{\theta_{\mathrm{old}}}$~\cite{schulman2017proximal}.
DAPO~\cite{yu2025dapo} removes the KL-divergence penalty from the reference policy and adopts a token-level entropy-aware clipping to further stabilize training, and all other aspects of the GRPO formulation are unchanged.

\textbf{Gradient informativeness and the learning zone.}
Analyzing the gradient of Eq.~\eqref{eq:grpo_obj} reveals a critical non-uniformity across prompt groups. Under the standard fixed-baseline approximation (treating $p_i^{(t)}$ as constant when differentiating), the gradient contribution from prompt~$i$ is proportional to
\begin{equation}
\label{eq:grpo_grad}
\nabla_\theta\mathcal{L}^{(i)}
\;\propto\;
\frac{1}{n}\sum_{k=1}^{n}
  \tilde{A}_{i,k}\,\nabla_\theta\log\pi_\theta(y_{i,k}\mid x_i).
\end{equation}
When $p_i^{(t)}\!\to\!1$, all rollouts succeed and every advantage
$\tilde{A}_{i,k}\!\approx\!0$, so the gradient vanishes regardless of the policy's score-function magnitude. Symmetrically, when $p_i^{(t)}\!\to\!0$, all rollouts fail and advantages again collapse.
As Theorem~\ref{thm:energy_gradient} formalizes, the variance of this gradient
estimator is approximately proportional to the Bernoulli variance $p_i(1-p_i)$,
which peaks at $p=0.5$ and vanishes at both extremes.
This identifies the \emph{learning zone}, the intermediate pass-rate regime, as the locus of maximum gradient informativeness, and motivates concentrating the training budget accordingly.

\textbf{Online data selection.}
At each training step~$t$, the \emph{online data selection problem} is to identify a subset $\mathcal{S}_t\subseteq\mathcal{D}$ with $|\mathcal{S}_t|=\lfloor\kappa N\rfloor$ ($\kappa\in(0,1)$) that maximizes expected policy improvement per unit of compute, using only information available at step~$t$.
A practical criterion must be (i)~updated online without offline pre-processing, (ii)~aligned with the gradient signal characterised by Eq.~\eqref{eq:grpo_grad}, and (iii)~computationally cheap, requiring no auxiliary models beyond the standard rollout phase.
\section{Methodology}
\label{sec:method}
This section presents the \textsc{LZE} framework in three parts. We first introduce the Learning-Zone Energy Score and explain the design rationale behind each of its three components (Section~\ref{sec:score}). We then provide theoretical support for the design by connecting the score to the GRPO gradient variance and classical signal-processing constructs (Section~\ref{sec:theory}). Finally, we describe how the Energy Score drives backward Top-$K$ selection and how a complementary forward pruner with replay reduces rollout generation cost (Section~\ref{sec:algorithm}).

\subsection{The Learning-Zone Energy Score}
\label{sec:score}
The key challenge in online data selection for RL post-training is to construct a criterion that remains adaptive to the policy's evolving capability while staying aligned with the underlying gradient signal. Motivated by the observation in Section~\ref{sec:background} that the GRPO gradient variance for prompt~$i$ scales with $p_i(1-p_i)$, we propose a score that captures three complementary aspects of prompt informativeness: intrinsic difficulty, current outcome uncertainty, and recent learning dynamics. Formally, for prompt~$i$ at training step~$t$, we define the \textbf{Learning-Zone Energy Score} as
\begin{equation}
\label{eq:energy}
E_i^{(t)} \;=\; \underbrace{D_i^{(0)}}_{\text{difficulty anchor}}
             \;\cdot\;\underbrace{4p_i^{(t)}\!\left(1-p_i^{(t)}\right)}_{\text{outcome uncertainty}}
             \;\cdot\;\underbrace{\left(1+\alpha\,m_i^{(t)}\right)}_{\text{momentum}},
\end{equation}
where the three multiplicative factors correspond to the three signals above and each plays a distinct and non-redundant role.

\textbf{Outcome uncertainty: $4p_i^{(t)}(1-p_i^{(t)})$.}
The central term is the normalized Bernoulli variance of the current pass rate, which lies in $[0,1]$, peaks at $p=0.5$, and vanishes at both extremes. This term directly implements the learning-zone intuition: it concentrates selection on prompt groups where outcomes are mixed, i.e., neither consistently correct nor consistently incorrect, and suppresses both extremes.
By Theorem~\ref{thm:energy_gradient}, the GRPO gradient variance for prompt~$i$ is approximately proportional to $p_i(1-p_i)\,G_i(\theta)$, confirming this term as the primary driver of gradient informativeness. Unlike Focal Loss~\cite{lin2017focal}, which uses $(1-p)^\gamma$ and penalises easy positives but not hard negatives, $4p(1-p)$ is \emph{symmetric}: it simultaneously down-weights both extremes, providing a balanced treatment of the two failure modes.

\textbf{Difficulty anchor: $D_i^{(0)}$.}
Not all prompts near $p=0.5$ are equally valuable. Two prompts may have the same current pass rate but very different histories: one that was initially easy and has since regressed conveys different information from one that was initially hard and is now being cracked for the first time. We record the pass rate at the start of training and set $D_i^{(0)}=1-p_i^{(0)}\in[0,1]$, encoding the intrinsic hardness of each prompt as a fixed prior that is never updated. This anchor prevents the curriculum from collapsing toward problems that were initially easy but happen to sit transiently near $p=0.5$ due to policy noise.

\textbf{Pass-rate momentum: $m_i^{(t)}$.}
Even with difficulty and uncertainty both captured, a prompt stagnating near $p=0.5$ for many steps contributes less than one where the policy is actively improving. We track learning dynamics via an exponential moving average $\mu_i^{(t)}=\lambda\mu_i^{(t-1)}+(1-\lambda)p_i^{(t)}$ with decay $\lambda\in(0,1)$, and define the momentum $m_i^{(t)}=p_i^{(t)}-\mu_i^{(t-1)}$, which is positive when the policy is improving and near zero when learning stagnates. The factor $(1+\alpha m_i)$ amplifies the improvement in prompts and is neutral otherwise. As formalized in Section~\ref{sec:theory}, $m_i^{(t)}$ is the output of a causal high-pass filter on the pass-rate sequence, capturing rapid changes while attenuating slow trends.

To provide intuition for why \textsc{LZE} improves over uniform training, Figure~\ref{fig:case_study} (Appendix~\ref{sec:case_study}) illustrates three representative prompt types: a trivially solved prompt ($p\approx1$) and an unsolvable prompt ($p\approx0$), both
filtered out due to near-zero Energy, and a frontier prompt ($p\approx0.5$) that receives the highest score and is selected for policy updates.

\subsection{Theoretical Analysis}
\label{sec:theory}
We now explain why the uncertainty term in the Energy Score is well motivated from the perspective of prompt-level gradient informativeness, and provide the signal-processing characterization of the momentum term. Throughout, let $p_i$ denote the true pass rate under the current policy, and let $G_i(\theta)\coloneqq\mathbb{E}_{y\sim\pi_\theta(\cdot|x_i)}[\|\nabla_\theta\log\pi_\theta(y\mid x_i)\|^2]$ denote the expected squared score-function norm.

\begin{theorem}[Uncertainty term and gradient variance]
\label{thm:energy_gradient}
Consider prompt~$i$ with $n$ i.i.d.\ rollouts and gradient estimator $g_i=n^{-1}\sum_{k}\tilde{A}_k\nabla_\theta\log\pi_\theta(y_k\mid x_i)$, where $\tilde{A}_k=r_k-p_i$ under the standard fixed-baseline approximation~\cite{williams1992simple}.
Assume the score-function norm is approximately homogeneous across reward outcomes as in Eq.~\eqref{eq:score_homogeneity}:
\begin{equation}
\label{eq:score_homogeneity}
\mathbb{E}\!\bigl[\|\nabla_\theta\log\pi_\theta(y\mid x_i)\|^2\mid r=1\bigr]
\;\approx\;
\mathbb{E}\!\bigl[\|\nabla_\theta\log\pi_\theta(y\mid x_i)\|^2\mid r=0\bigr]
= G_i(\theta).
\end{equation}
Then Eq.~\eqref{eq:grad_var} holds:
\begin{equation}
\label{eq:grad_var}
\mathrm{Var}(g_i)
\;\approx\;
\frac{p_i(1-p_i)}{n}\;G_i(\theta),
\end{equation}
where the approximation holds when $\|\nabla_\theta p_i\|^2\ll p_i(1-p_i)G_i(\theta)$. The variance is maximized at $p_i=0.5$ and vanishes at both $p_i=0$ and $p_i=1$. Consequently, $E_i$ is monotonically related to $\mathrm{Var}(g_i)$, modulated by the difficulty anchor and momentum.
\end{theorem}
\vspace{-8pt}
\begin{proof}
Under the fixed-baseline assumption, $\tilde{A}_k=r_k-p_i$ with $p_i$ fixed. Since rollouts $y_1,\ldots,y_n$ are i.i.d., the per-sample gradient estimates $\hat{g}_k=\tilde{A}_k\,\nabla_\theta\log\pi_\theta(y_k\mid x_i)$ are mutually independent. By the variance decomposition for i.i.d.\ summands,
\begin{equation}
\label{eq:grad_derivation}
\mathrm{Var}(g_i)
=\frac{1}{n}\,\mathrm{Var}(\hat{g}_1)
=\frac{1}{n}\!\left(
  \mathbb{E}\!\left[\|\hat{g}_1\|^2\right]-\|\mathbb{E}[\hat{g}_1]\|^2
\right).
\end{equation}
For the first term, using the homogeneity approximation in Eq.~\eqref{eq:score_homogeneity},
\begin{equation}
\label{eq:hatg_sq}
\mathbb{E}[\|\hat{g}_1\|^2]
=\mathbb{E}[\tilde{A}_1^2\|\nabla_\theta\log\pi_\theta(y_1|x_i)\|^2]
\approx \mathbb{E}[\tilde{A}_1^2]\,G_i(\theta)
=p_i(1-p_i)\,G_i(\theta).
\end{equation}
For the second term $\mathbb{E}[\hat{g}_1]=\nabla_\theta p_i(\theta)$, so
$\|\mathbb{E}[\hat{g}_1]\|^2=\|\nabla_\theta p_i\|^2$.
Substituting Eq.~\eqref{eq:hatg_sq} and the identity $\|\mathbb{E}[\hat{g}_1]\|^2=\|\nabla_\theta p_i\|^2$ into Eq.~\eqref{eq:grad_derivation} gives Eq.~\eqref{eq:grad_var_expansion}:
\begin{equation}
\label{eq:grad_var_expansion}
\mathrm{Var}(g_i)
=\frac{p_i(1-p_i)\,G_i(\theta)-\|\nabla_\theta p_i(\theta)\|^2}{n}
\;\approx\;
\frac{p_i(1-p_i)}{n}\,G_i(\theta),
\end{equation}
where the last step drops the subdominant term under the stated condition. Hence the Bernoulli variance $p_i(1-p_i)$ governs the dominant component of the prompt-level gradient variance, and the normalization $4p(1-p)\in[0,1]$ is a monotone rescaling with maximum at $p=0.5$.
\end{proof}
\vspace{-8pt}
\noindent\textbf{Momentum as a temporal change signal.}
The momentum $m_i^{(t)}=p_i^{(t)}-\mu_i^{(t-1)}$ measures how much the current pass rate deviates from recent history.
To see why large $|m_i|$ signals high gradient magnitude, consider the first-order Taylor expansion of $\ell_i(\theta)=p_i(\theta)$ around
$\theta_{t-1}$,
\begin{equation}
\label{eq:taylor_momentum}
\Delta p_i^{(t)} := p_i^{(t)}-p_i^{(t-1)}
\approx \bigl(\nabla_\theta \ell_i(\theta_{t-1})\bigr)^\top\!\Delta\theta,
\end{equation}
where $\Delta\theta=\theta_t-\theta_{t-1}$. Eq.~\eqref{eq:taylor_momentum}, combined with Exponential Moving Average (EMA) smoothing, implies $|m_i^{(t)}|\lesssim\|\nabla_\theta p_i\|\,\|\Delta\theta\|$ by Cauchy--Schwarz, so the momentum factor amplifies prompts where the policy is changing rapidly.

\begin{proposition}[EMA--convolution duality]
\label{prop:ema_conv}
With $\mu_i^{(0)}=p_i^{(0)}$, the EMA $\mu_i^{(t)}=\lambda\mu_i^{(t-1)}+(1-\lambda)p_i^{(t)}$
equals the causal convolution $(h*p_i)(t)$ with kernel $h(\tau)=(1-\lambda)\lambda^\tau\mathbf{1}_{\tau\geq0}$. Consequently, $m_i^{(t)}=p_i^{(t)}-\mu_i^{(t-1)}$ is the output of a causal \emph{high-pass filter} applied to $\{p_i^{(s)}\}_{s\leq t}$.
\end{proposition}
\vspace{-8pt}
\begin{proof}
Unrolling the recurrence gives $\mu_i^{(t)}=\sum_{\tau=0}^{t}(1-\lambda)\lambda^\tau p_i^{(t-\tau)}=(h*p_i)(t)$. Since $m_i^{(t)}=p_i^{(t)}-\mu_i^{(t-1)}$, the momentum equals the original signal minus a one-step-delayed smoothed version of itself, which is the output of the
complementary high-pass filter isolating rapid temporal changes.
\end{proof}

\paragraph{Interpretation via attention.}
\label{rem:attention}
The Energy Score can be read, by analogy, as a sample-level attention weight: $D_i^{(0)}$ acts as a key encoding intrinsic hardness, $\mu_i^{(t)}$ as a query representing current model state, and $4p(1-p)(1+\alpha m)$ as the normalized score, with Top-$K$ implementing hard attention. Compared to Focal Loss~\cite{lin2017focal} (asymmetric in $p$, static) and OHEM~\cite{shrivastava2016training} (loss-based, memoryless), $E_i$ is \emph{symmetric} in $p$, \emph{temporally aware} via EMA, and \emph{gradient-aligned} by Theorem~\ref{thm:energy_gradient}.

\begin{algorithm}[t]
\caption{Learning-Zone Energy Training }
\label{alg:lze}
\begin{algorithmic}[1]
\Require Dataset $\mathcal{D}$, policy $\pi_{\theta_0}$, selection ratio $\kappa$,
         EMA decay $\lambda$, momentum weight $\alpha$,
         prune threshold $T_{\mathrm{prune}}$, replay ratio $\rho$
\State \textbf{Initialise:}
       Run a single forward pass over $\mathcal{D}$; for each $i$, compute
       $p_i^{(0)}$, set $D_i^{(0)}\leftarrow1-p_i^{(0)}$,
       $\mu_i^{(0)}\leftarrow p_i^{(0)}$,
       epoch-correct counter $c_i\leftarrow0$
\State Active pool $\mathcal{A}\leftarrow\mathcal{D}$;\enspace
       Prune pool $\mathcal{P}\leftarrow\emptyset$
\For{each epoch $e=1,2,\ldots$}
  \State \textbf{Replay:}
         Sample $\lfloor\rho|\mathcal{P}|\rfloor$ prompts from $\mathcal{P}$;
         run rollouts; move any with $p_i<1$ back to $\mathcal{A}$
  \For{each training step $t$ within epoch $e$}
    \State Collect rollouts for all $i\in\mathcal{A}$;
           compute $p_i^{(t)}$
    \State $\mu_i^{(t)}\leftarrow\lambda\,\mu_i^{(t-1)}+(1-\lambda)\,p_i^{(t)}$
           \hfill\Comment{EMA update}
    \State $m_i^{(t)}\leftarrow p_i^{(t)}-\mu_i^{(t-1)}$
           \hfill\Comment{Momentum}
    \State $E_i^{(t)}\leftarrow
           D_i^{(0)}\cdot4p_i^{(t)}(1-p_i^{(t)})\cdot(1+\alpha\,m_i^{(t)})$
           \hfill\Comment{Energy Score}
    \State $\mathcal{S}_t\leftarrow
           \mathrm{Gumbel\text{-}Top\text{-}K}\!\bigl(\{E_i^{(t)}\}_{i\in\mathcal{A}},
           \lfloor\kappa|\mathcal{A}|\rfloor\bigr)$
           \hfill\Comment{Backward selection}
    \State $\theta_{t+1}\leftarrow\mathrm{GradStep}(\pi_{\theta_t},\mathcal{S}_t)$
           \hfill\Comment{Policy update on selected subset}
  \EndFor
  \State \textbf{Forward pruning:}
         For each $i\in\mathcal{A}$, increment $c_i$ if $p_i^{(t_{\rm last})}=1$,
         else reset $c_i\leftarrow0$;
         if $c_i\geq T_{\mathrm{prune}}$, move $i$ from $\mathcal{A}$ to $\mathcal{P}$
\EndFor
\end{algorithmic}
\end{algorithm}

\subsection{Prompt Selection and Training Algorithm}
\label{sec:algorithm}
\textbf{Backward selection.}
At each training step~$t$, \textsc{LZE} first collects rollouts for all prompts in the active pool, computes pass rates and Energy Scores via Eq.~\eqref{eq:energy}, and then retains the top-$\kappa$ fraction for policy optimization. To balance exploitation with exploration, small Gumbel perturbations are added before ranking: $\tilde{E}_i=E_i^{(t)}+g_i$, $g_i\sim\mathrm{Gumbel}(0,1)$. The top-$K=\lfloor\kappa|\mathcal{A}|\rfloor$ prompts by perturbed score form $\mathcal{S}_t$; only these groups participate in policy gradient updates.
Backward selection mainly reduces \emph{optimization cost}, and rollouts are still collected over the full active pool to update Energy Scores.

\textbf{Forward pruning with replay.}
At the epoch level, if prompt~$i$ achieves full correctness ($p_i=1$) for
$T_{\mathrm{prune}}=2$ consecutive epochs, it is moved to a prune pool~$\mathcal{P}$ and skipped in future rollout generation. At the start of each epoch, a fraction~$\rho$ of $\mathcal{P}$ is re-evaluated, and any prompt that no longer achieves full correctness is restored to the active pool. The two stages are complementary: the backward selector concentrates the \emph{gradient} budget on the most informative groups, while the forward pruner eliminates \emph{rollout} cost for persistently solved prompts. The complete training procedure is presented in Algorithm~\ref{alg:lze}.

\section{Experiments}
\label{sec:experiments}
We evaluate \textsc{LZE} on mathematical reasoning benchmarks across multiple model scales, comparing against representative data selection methods spanning offline filtering, online and SFT-based rejection. We further analyze training efficiency and convergence dynamics and conduct ablation studies to verify the contribution of each Energy Score component and the sensitivity to key hyperparameters.

\textbf{Models and Datasets.}
We evaluate our method on mathematical reasoning, where models must produce step-by-step Chain-of-Thought (CoT) reasoning together with a boxed final answer.
To ensure broad coverage, we experiment with models from the Qwen2.5~\citep{qwen2024qwen25} and Qwen3~\citep{qwen2025qwen3} families (the latter with thinking mode enabled throughout training), and the full list of evaluated variants is given in Table~\ref{tab:pass1_all}.
Ablation studies are primarily conducted with Qwen2.5-Math-1.5B and -7B~\citep{qwen2024math}, chosen for their strong performance ceiling and stable inference behaviour, making them reliable starting points for post-training.
We train on three datasets of increasing difficulty: GSM8K~\citep{cobbe2021training}, MATH~\citep{hendrycks2021math}, and DAPO-MATH (English split)~\citep{yu2025dapo}. More details are provided in Appendix~\ref{sec:model_data_details}.

\textbf{Hyper-parameters.}
\label{par:hyperparams}
We follow the default configurations of the Verl framework~\citep{sheng2024hybridflow} for PPO and GRPO training.
Following recent practice in group-based RL post-training~\citep{yu2025dapo,he2025justrl}, we remove the KL-divergence penalty and use temperature $=1.0$, top-$p=1.0$, and top-$k=-1$ to encourage diverse rollout generation.
All runs use AdamW with learning rate $1\times10^{-6}$.
The prompt batch size is 1024 for models under 4B parameters and 512 for larger
models, with $n=8$ rollouts per prompt.
Maximum generation length is 4096 tokens for Qwen2.5-family models and 16384 tokens for Qwen3-family models.
The selection ratio $k{=}0.4$ and momentum coefficient $\alpha=0.3$ are fixed
across all main experiments, and sensitivity analyses are in Section~\ref{sec:ablation}.

\textbf{Evaluation.}
\label{par:evaluation}
We use a standardized chat template (Appendix~\ref{sec:chat_template}) and determine correctness via rule-based answer matching using \texttt{math-verify}.
We report Pass@1 accuracy on both in-domain and out-of-distribution (OOD)
benchmarks, together with estimated FLOPs budgets, wall-clock training time, and convergence speed.

\begin{table*}[t]
\centering
\small
\setlength{\extrarowheight}{3pt}
\setlength{\tabcolsep}{5pt}
\renewcommand{\arraystretch}{1.2}
\vspace{-10pt}
\caption{Main Performance \& Efficiency. Pass@1 comparison on in-domain and OOD evaluation sets, alongside estimated computational savings.}
\label{tab:pass1_all}
\resizebox{\textwidth}{!}{
\begin{tabular}{c c l c c c c c c c}
\toprule
\multirow{2}{*}{\textbf{Training Set}} & \multirow{2}{*}{\textbf{Model}} & \multirow{2}{*}{} & \multicolumn{6}{c}{\textbf{Accuracy (Pass@1)}} & \multirow{2}{*}{\textbf{FLOPs Budget}} \\
\cmidrule(lr){4-9}
 & & & \textbf{GSM8K} & \textbf{MATH} & \textbf{DAPO-MATH} & \textbf{AMC23} & \textbf{AIME25} & \textbf{Avg} & \\
\midrule
\multirow{4}{*}{GSM8K}
& \multirow{2}{*}{Qwen2.5-Math-1.5B}
  & Base & 0.859 & 0.538 & 0.249 & 0.575 & 0.111 & 0.466 & $5.47\times10^{19}$ \\
& & Ours & \best{0.860} & \best{0.553}\up{+2.8\%} & \best{0.265}\up{+6.4\%} & \best{0.600}\up{+4.3\%} & \best{0.122}\up{+9.9\%} & \best{0.480}\up{+2.9\%} & $3.50\times10^{19}$ \\
\cmidrule(lr){2-10}
& \multirow{2}{*}{Qwen3-1.7B}
  & Base & 0.904 & 0.573 & 0.450 & 0.600 & 0.211 & 0.548 & $4.94\times10^{19}$ \\
& & Ours & \best{0.912} & \best{0.578} & \best{0.452} & \best{0.700}\up{+16.7\%} & \best{0.267}\up{+26.5\%} & \best{0.582}\up{+6.2\%} & $3.17\times10^{19}$ \\
\midrule
\multirow{6}{*}{MATH}
& \multirow{2}{*}{Qwen2.5-Math-1.5B}
  & Base & 0.807 & \best{0.583} & 0.281 & 0.550 & 0.122 & 0.469 & $5.82\times10^{19}$ \\
& & Ours & \best{0.825}\up{+2.2\%} & \best{0.583} & \best{0.310}\up{+10.3\%} & \best{0.650}\up{+18.2\%} & \best{0.178}\up{+45.9\%} & \best{0.509}\up{+8.6\%} & $3.73\times10^{19}$ \\
\cmidrule(lr){2-10}
& \multirow{2}{*}{Qwen2.5-Math-7B}
  & Base & 0.868 & 0.621 & 0.440 & 0.700 & 0.278 & 0.581 & $1.70\times10^{20}$ \\
& & Ours & 0.876 & \best{0.628} & \best{0.462}\up{+3.1\%} & \best{0.725}\up{+3.6\%} & \best{0.311}\up{+11.9\%} & \best{0.600}\up{+3.2\%} & $1.09\times10^{20}$ \\
\cmidrule(lr){2-10}
& \multirow{2}{*}{Qwen3-4B}
  & Base & \best{0.945} & 0.771 & 0.576\downnote & 0.875 & 0.544 & 0.742 & $1.18\times10^{20}$ \\
& & Ours & 0.942 & \best{0.775} & 0.576\downnote & \best{0.900}\up{+2.9\%} & \best{0.556}\up{+2.2\%} & \best{0.750}\up{+1.0\%} & $7.55\times10^{19}$ \\
\midrule
\multirow{6}{*}{DAPO-MATH}
& \multirow{2}{*}{Qwen2.5-Math-7B}
  & Base & \best{0.928} & 0.699 & 0.530 & 0.775 & 0.344 & 0.655 & $2.94\times10^{20}$ \\
& & Ours & \best{0.928} & \best{0.708}\up{+1.3\%} & \best{0.546}\up{+3.0\%} & \best{0.825}\up{+6.5\%} & \best{0.356}\up{+3.5\%} & \best{0.673}\up{+2.7\%} & $1.89\times10^{20}$ \\
\cmidrule(lr){2-10}
& \multirow{2}{*}{Qwen3-4B}
  & Base & 0.914 & 0.630 & \best{0.807} & \best{0.925} & 0.544 & 0.764 & $8.56\times10^{19}$ \\
& & Ours & \best{0.941}\up{+3.0\%} & \best{0.639}\up{+1.4\%} & 0.795 & 0.900 & \best{0.556}\up{+2.2\%} & \best{0.766}\up{+0.3\%} & $5.48\times10^{19}$ \\
\cmidrule(lr){2-10}
& \multirow{2}{*}{Qwen3-8B}
  & Base & 0.906 & 0.626 & 0.768 & 0.825 & 0.456 & 0.716 & $9.85\times10^{19}$ \\
& & Ours & \best{0.911} & \best{0.630} & \best{0.807}\up{+5.1\%} & \best{0.925}\up{+12.1\%} & \best{0.500}\up{+9.6\%} & \best{0.755}\up{+5.4\%} & $6.30\times10^{19}$ \\
\bottomrule
\end{tabular}
}
\vspace{-10pt}
\end{table*}

\subsection{Main Results}
\textbf{Accuracy and out-of-distribution generalization.}
Table~\ref{tab:pass1_all} summarizes the Pass@1 accuracy across both in-domain and out-of-distribution (OOD) mathematical benchmarks. We report \textit{Base} as the strong baseline without data filtering. Our method consistently outperforms the full-data baseline in nearly all configurations while training on only 40\% of prompts per step, confirming that concentrating the gradient budget on the learning frontier is more effective than uniform allocation. The gains are most pronounced on the challenging OOD benchmarks AMC\,23 and AIME\,25, which measure transferable mathematical reasoning rather than in-distribution recall. Representative gains include: Qwen2.5-Math-1.5B on MATH achieving $+45.9\%$
on AIME\,25 and $+18.2\%$ on AMC\,23; Qwen3-1.7B on GSM8K achieving $+26.5\%$
on AIME\,25 and $+16.7\%$ on AMC\,23; and Qwen3-8B on DAPO-MATH achieving
$+12.1\%$ on AMC\,23 and $+9.6\%$ on AIME\,25.

These OOD improvements are consistent with the theoretical prediction of
Theorem~\ref{thm:energy_gradient}: by suppressing both trivially solved and
completely failed prompt groups, \textsc{LZE} prevents the policy from
overfitting to the in-distribution difficulty profile, driving capability
gains that transfer more broadly. The gains are particularly pronounced for smaller models (1.5B and 1.7B), where the learning zone is narrower and the Energy Score's discriminative power is highest. Furthermore, consistent improvements across all scales from 1.5B to 8B and across both the Qwen2.5 and Qwen3 families confirm that learning-zone targeting is broadly applicable and not tied to any particular architecture or capacity regime.

\begin{wrapfigure}{r}{0.38\textwidth}
\vspace{-6pt}
\centering
\includegraphics[width=\linewidth]{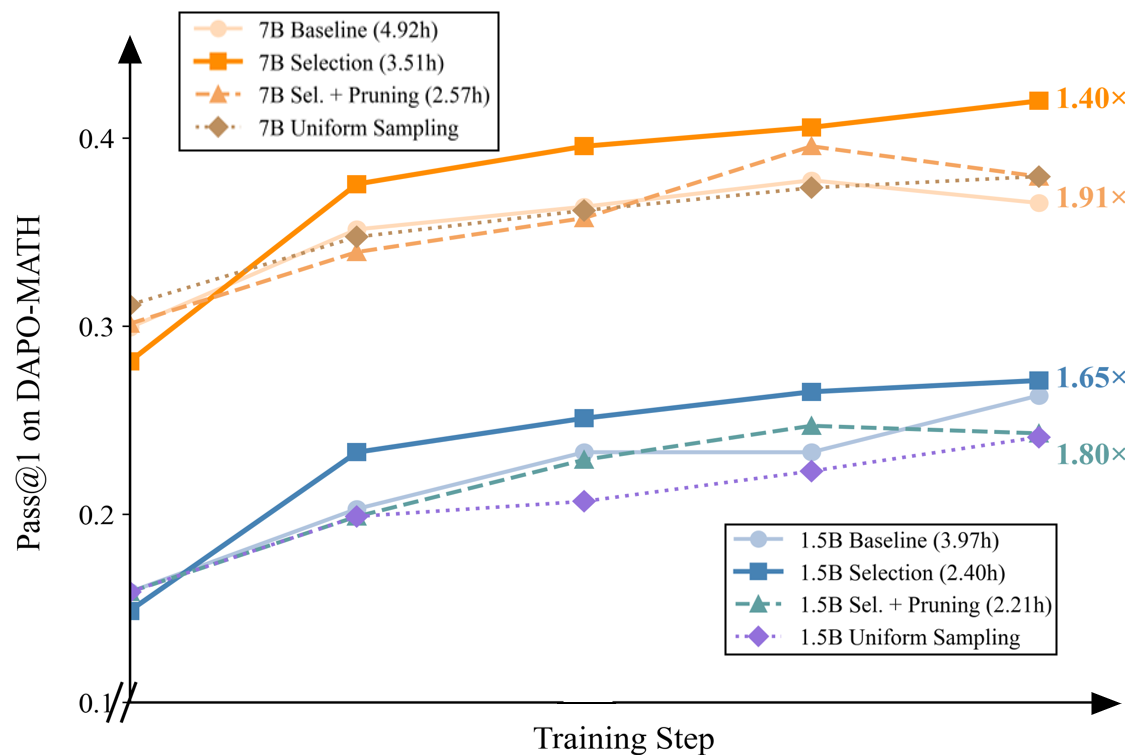}
\vspace{-10pt}
\caption{\small Pass@1 vs.\ wall-clock time for Qwen2.5-Math models. The $\times$ annotations denote speedup multipliers relative to the full-data baseline.}
\label{fig:learning_dynamics}
\vspace{-10pt}
\end{wrapfigure}
\textbf{Efficiency and training progress.}
The proposed data selection method reduces theoretical training FLOPs by approximately 36\% by restricting backpropagation to the top-$40\%$ fraction of prompt groups, while maintaining full rollout coverage to recompute Energy Scores at each step (see Appendix~\ref{sec:flops_derivation} for the derivation). Figure~\ref{fig:learning_dynamics} plots Pass@1 against wall-clock training time.
For Qwen2.5-Math-1.5B, the baseline stabilizes at 3.97 hours; \textsc{LZE} reaches better performance in 2.40 hours ($1.65\times$ faster). For the 7B model, convergence is reached in 3.51 hours versus 4.92 hours for the baseline ($1.40\times$ faster). Adding the forward pruner yields further speedups of $1.80\times$ and $1.91\times$ for the 1.5B and 7B models, respectively, at a modest accuracy trade-off relative to the backward-only variant. Notably, uniform sampling at the same 40\% retention ratio consistently underperforms the full-data baseline and even falls below the forward pruner in final accuracy, confirming that the gains of \textsc{LZE} stem from the principled identification of learning-zone prompts rather than from any implicit regularization effect of reduced data volume.

\begin{table*}[t]
\centering
\small
\setlength{\extrarowheight}{3pt}
\setlength{\tabcolsep}{5pt}
\renewcommand{\arraystretch}{1.2}
\vspace{-8pt}
\caption{Comparison of Pass@1 performance between the proposed method and existing data filtering baselines on the MATH dataset. All reported metrics represent the highest Pass@1 score achieved across the entire training trajectory. ``rem.'' stands for remove.}
\label{tab:comparison_methods}
\resizebox{\textwidth}{!}{
\begin{tabular}{l c c c c c c c c}
\toprule
\multirow{2}{*}{\textbf{Method}} & \multicolumn{4}{c}{\textbf{Qwen2.5-Math-1.5B}} & \multicolumn{4}{c}{\textbf{Qwen2.5-Math-7B}} \\
\cmidrule(lr){2-5} \cmidrule(lr){6-9}
 & \textbf{DAPO-MATH} & \textbf{AMC23} & \textbf{AIME25} & \textbf{Avg Score} & \textbf{DAPO-MATH} & \textbf{AMC23} & \textbf{AIME25} & \textbf{Avg Score} \\
\midrule
Baseline                                 & 0.281 & 0.575 & 0.122 & 0.326    & 0.448 & 0.700 & 0.278 & 0.475 \\
Ours                                     & 0.310 & \best{0.650} & \best{0.178} & \best{0.379}   & \best{0.462} & \best{0.725} & \best{0.311} & \best{0.500} \\
Reinforce-Rej (with online filtering)    & 0.235 & 0.525 & 0.122 & 0.294    & 0.351 & \best{0.725} & 0.200 & 0.425 \\
Pre-Filter (rem.\ all wrong)             & \best{0.327} & 0.575 & 0.122 & 0.341    & 0.447 & 0.675 & 0.244 & 0.465 \\
Pre-Filter (rem.\ all right)             & 0.305 & 0.575 & 0.167 & 0.349    & 0.458 & 0.675 & 0.279 & 0.471 \\
Pre-Filter (discard all)                 & 0.303 & 0.600 & 0.156 & 0.353    & 0.428 & \best{0.725} & 0.244 & 0.466 \\
RAFT                                     & 0.281 & 0.575 & 0.133 & 0.330    & 0.363 & 0.575 & 0.189 & 0.348 \\
Uniform sampling                         & 0.273 & 0.550 & 0.122 & 0.315    & 0.437 & 0.675 & 0.289 & 0.467 \\
\bottomrule
\end{tabular}
}
\vspace{-8pt}
\end{table*}

\textbf{Comparison with existing methods.} Table~\ref{tab:comparison_methods} compares \textsc{LZE} against offline, online, and SFT-based baselines on the MATH dataset, and empirical results suggest:
\begin{itemize}[leftmargin=*]
  \item \emph{Offline filtering is ceiling-limited by an off-policy data distribution.}
        All Pre-Filter variants improve marginally over the baseline, but their gains plateau as training progresses. Static filters commit to a snapshot of model capability at initialization, but as the policy evolves, the curated distribution becomes increasingly off-policy and can no longer track the actual learning frontier. \textsc{LZE} recomputes its Energy Score at every training step, keeping the retained distribution continuously on-policy.

  \item \emph{Binary online rejection destabilizes training.}
        Reinforce-Rej~\cite{wei2025minimalist} permanently removes extreme-probability groups at first detection, disrupting the continuity of the training distribution and inducing optimization instability, as evidenced by its substantially lower OOD performance.
        \textsc{LZE} replaces hard removal with a continuous Energy Score for soft
        re-weighting, maintaining a stable training distribution without sacrificing adaptivity.
  \item \emph{SFT-based filtering cannot substitute for RL exploration.}
        RAFT~\cite{dong2023raft} retains only positively rewarded rollouts for supervised training and underperforms the full RL baseline on all benchmarks. This confirms that policy-gradient-driven exploration is essential for generalization: the diversity of the gradient signal cannot be recovered by SFT alone.
\end{itemize}

\begin{figure}[H]
  \centering
  \begin{subfigure}[t]{0.48\linewidth}
    \centering
    \includegraphics[width=0.97\linewidth]{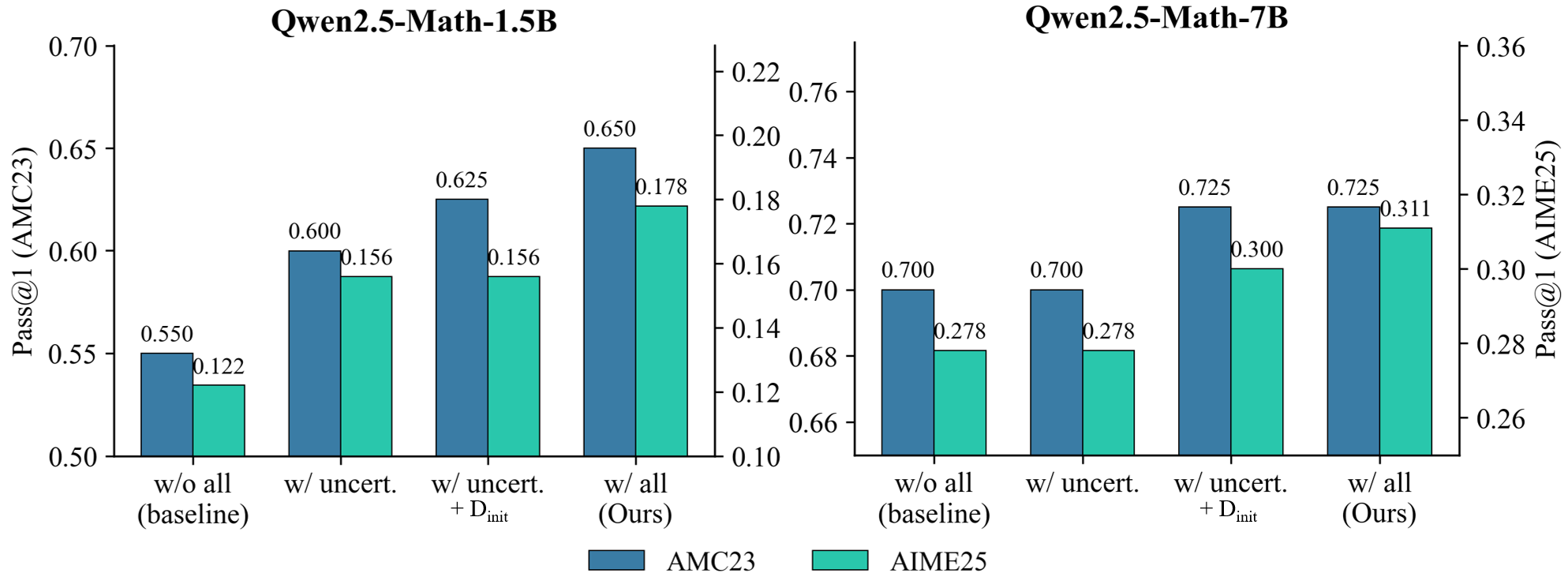}
    \caption{Component ablation.}
    \label{fig:ablation_components}
  \end{subfigure}
  \hspace{0.\linewidth}
  \begin{subfigure}[t]{0.23\linewidth}
    \centering
    \includegraphics[width=0.97\linewidth]{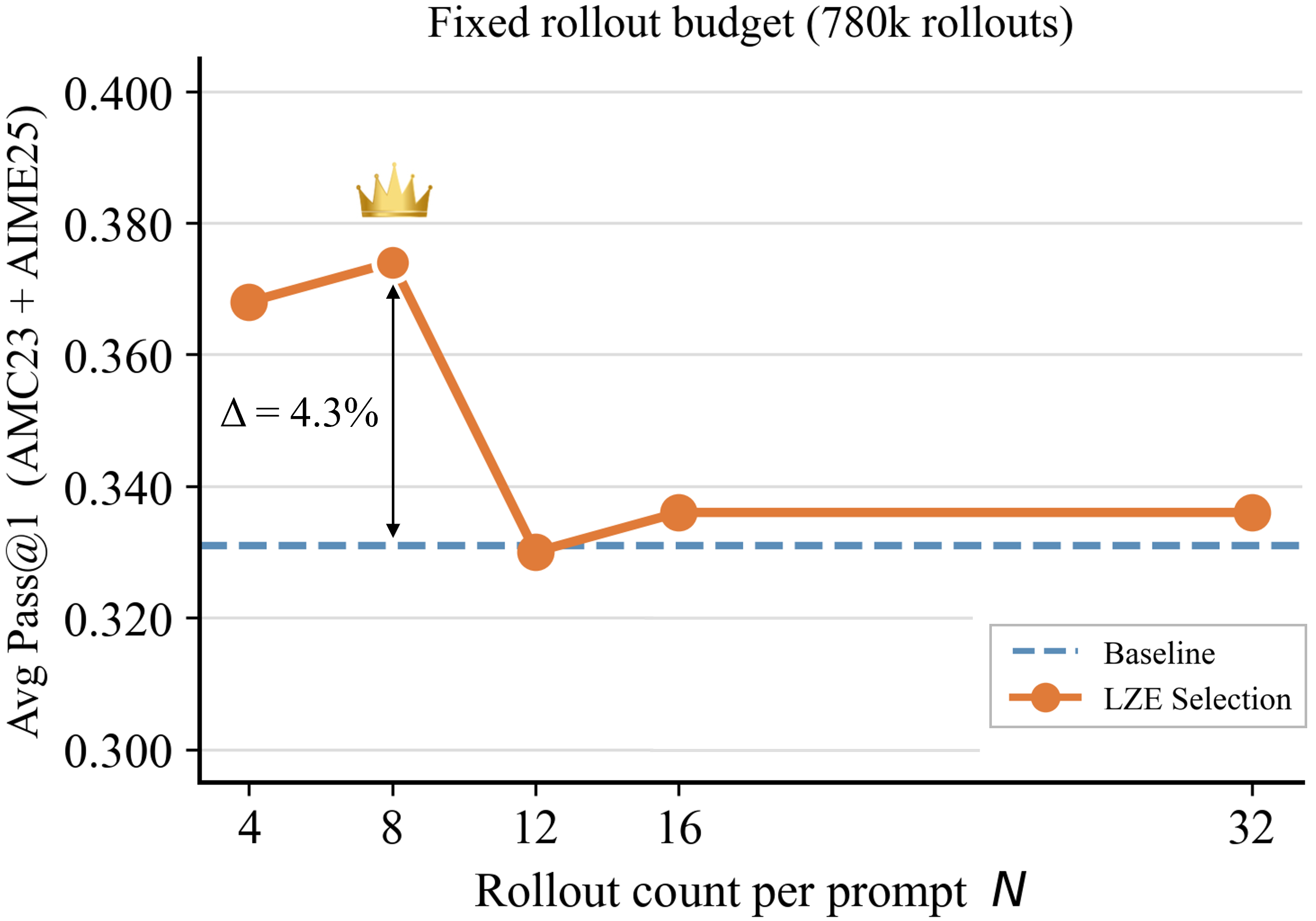}
    \caption{Rollout sensitivity.}
    \label{fig:ablation_rollout_n}
  \end{subfigure}
  \hspace{0.\linewidth}
  \begin{subfigure}[t]{0.26\linewidth}
    \centering
    \includegraphics[width=0.8\linewidth]{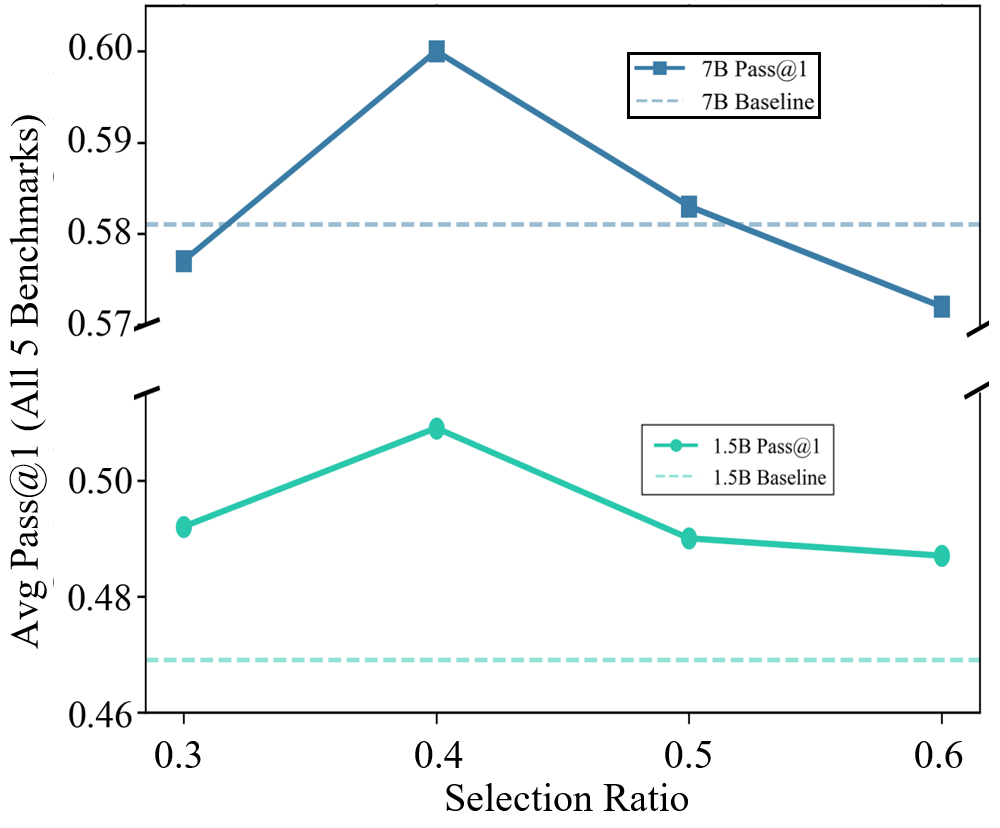}
    \caption{Selection-ratio sensitivity.}
    \label{fig:ablation_ratio}
  \end{subfigure}
  \caption{Ablation studies on Energy Score components, rollout count $n$, and selection ratio $\kappa$, conducted on Qwen2.5-Math models trained on MATH.}
  \label{fig:ablation_full}
\vspace{-8pt}
\end{figure}

\subsection{Ablation Study}
\label{sec:ablation}
\textbf{Contribution of each Energy Score component.}
We ablate each factor in Eq.~\eqref{eq:energy} by removing it in turn and measuring the resulting drop in Pass@1 on AMC23 and AIME25 using Qwen2.5-Math-1.5B and Qwen2.5-Math-7B. As shown in Figure~\ref{fig:ablation_components}, removing the uncertainty term $4p(1-p)$ causes the largest degradation, directly validating the learning-zone hypothesis of Theorem~\ref{thm:energy_gradient}: gradient informativeness is maximized at $p\approx0.5$ and the uncertainty term is the primary driver of the improvement. Removing the difficulty anchor $D_{\mathrm{init}}$ allows the selector to collapse onto initially easy prompts that sit transiently near $p=0.5$ due to policy noise, reducing the quality of the retained set. Besides, removing the momentum term causes the selector to treat stagnating and actively improving prompts equivalently, weakening OOD generalization.

\textbf{Sensitivity to rollout $n$.}
Under a fixed total rollout budget of $780\text{k}$ (corresponding to $n{=}8$ over approximately 7 epochs), we vary the per-prompt rollout $n\in\{4,8,12,16,32\}$ for both the baseline and \textsc{LZE}, holding all
other hyperparameters constant. As shown in Figure~\ref{fig:ablation_rollout_n}, \textsc{LZE} consistently outperforms the baseline across all values of $n$, with peak absolute performance at $n{=}8$.
Smaller $n$ reduces the number of samples used to estimate the group-level pass rate, making the Energy Score noisy and less discriminative. Larger $n$ improves per-prompt accuracy but reduces the number of epochs within the fixed budget, limiting how many times the score can adapt to policy changes. The default $n{=}8$ strikes the best balance between estimate quality and adaptation frequency, and the performance gap between \textsc{LZE} and the baseline is most pronounced at this setting.
Unconstrained upper-bound results are provided in Appendix~\ref{sec:appendix_ablation}, which confirm that $n{=}8$ remains a robust default across both efficiency-constrained and quality-maximizing regimes.

\textbf{Sensitivity to selection ratio $k$.}
Figure~\ref{fig:ablation_ratio} reports the averaged performance under varying $\kappa$ across both the 1.5B and 7B models. and Pass@1 peaks at $\kappa=0.4$. At lower ratios, the policy gradient becomes sample-starved and updates grow brittle. At higher ratios, prompts from the all-correct and all-incorrect regimes re-enter the training batch, diluting the gradient signal with uninformative updates, consistent with the prediction of Theorem~\ref{thm:energy_gradient} that $\mathrm{Var}(g_i)$ vanishes at both extremes.

\vspace{-5pt}
\section{Related Work}
\vspace{-5pt}
\paragraph{Reinforcement Learning for Mathematical Reasoning.}
Building on chain-of-thought reasoning~\cite{wei2022chain}, self-consistency decoding~\cite{wang2023selfconsistency}, and tree search~\cite{yao2023tree,zelikman2022star}, reinforcement learning has become the dominant post-training paradigm for LLM reasoning. PPO~\cite{schulman2017proximal} served as the canonical optimizer in early RLHF work~\cite{ouyang2022training}, but REINFORCE-style variants have since gained traction for lower complexity and favourable scaling~\cite{williams1992simple,ahmadian2024back,kool2019buy}. GRPO~\cite{deepseek2024deepseekmath,deepseek2025deepseekr1} normalizes rewards within each prompt group to estimate advantages without a critic, and DAPO~\cite{yu2025dapo} improves stability through dynamic sampling and entropy-preserving clipping. Recent analysis suggests that the reasoning capacity of RL with verifiable rewards remains bounded by the base model's distribution~\cite{liu2025understanding}, and that data selection can matter as much as the optimizer~\cite{wei2025minimalist,mao2026dynamicspredictive}.

\textbf{Curriculum Learning and Data Selection.}
Curriculum learning~\cite{bengio2009curriculum} and self-paced learning~\cite{kumar2010self} established that organizing examples by difficulty improves optimization. In LLM post-training, data selection strategies prioritizing instances by quality, difficulty, or diversity have been shown to outperform naive quantity scaling~\cite{sorscher2022beyond,liu2024what}. Difficulty-aware approaches use proxies such as instruction complexity~\cite{lu2024instag}, model uncertainty~\cite{kuhn2023semantic}, focal weighting~\cite{lin2017focal}, and online hard-example mining~\cite{shrivastava2016training} to concentrate the gradient budget on informative samples. Diversity-oriented methods complement these via semantic deduplication, quality-driven subset selection, and gradient-based influence estimation~\cite{abbas2023semdedup,zhou2023lima,chen2024alpagasus,xia2024less,xie2023data}. Token-level selection further refines this by allocating training signal only to informative tokens within each sequence~\cite{lin2024rho1}. Most of these approaches operate offline and cannot adapt to the evolving policy landscape of online RL.

\textbf{Data Filtering in RL Post-Training.}
A growing body of work investigates data filtering during RL post-training, motivated by the observation that not all rollouts contribute equally to policy improvement. Offline pre-filtering removes prompts that are too easy or too hard via a one-time difficulty scoring pass~\citep{singh2024beyond,gulcehre2023rest}, but cannot adapt as the policy evolves. Online rejection methods, including Best-of-$N$ sampling~\citep{cobbe2021training,yuan2023scaling,yuan2024self}, iterative self-training~\citep{dong2023raft,pang2024iterative}, and Reinforce-Rej~\citep{wei2025minimalist}, operate adaptively but apply binary keep-or-drop decisions that discard the continuous gradient signal. Reward- and signal-driven re-weighting methods use curiosity signals or mistake patterns~\citep{pathak2017curiosity,wang2023learning}; process-level verification provides finer-grained filtering~\citep{uesato2022solving,lightman2024lets}, while generative reward modeling offers an alternative verifier design~\citep{zhang2025generative}, all at the cost of auxiliary models. DPS~\cite{mao2026dynamicspredictive}, a recent method, reduces redundant rollouts by predicting prompt solvability via a hidden Markov model, but necessitates online Bayesian inference and a learned generative state model, inevitably incurring additional overhead. \textsc{LZE} differs from all of the above: the Energy Score is a continuous, closed-form criterion requiring no auxiliary models, and naturally subsumes binary filtering as a special case.
\vspace{-5pt}
\section{Conclusion}
\vspace{-5pt}
In this work, we propose \textbf{Learning-Zone Energy (LZE)}, a fully online data selection framework for efficient RL post-training. The central insight is that GRPO gradient variance is approximately proportional to the pass-rate variance $p(1-p)$, which peaks in the intermediate regime where the model neither consistently succeeds nor consistently fails. Building on this, the Learning-Zone Energy Score fuses \textit{initial difficulty}, \textit{outcome uncertainty}, and \textit{pass-rate momentum} into a closed-form criterion updated at every training step without auxiliary models. Empirical results demonstrate that \textsc{LZE} achieves superior performance on nearly all benchmarks while training on only 40\% of prompts per step, with particularly pronounced gains on OOD settings.
One limitation of the current study is that all experiments focus on mathematical reasoning tasks with binary verifiable rewards, so whether the same scoring rule transfers to open-ended or noisy-reward settings remains an important direction for future work.

\bibliographystyle{plainnat}
\bibliography{neurips_2026}

@misc{he2025justrl,
  title={{JustRL}: Scaling a 1.5B {LLM} with a Simple {RL} Recipe},
  author={Bingxiang He and Zekai Qu and Zeyuan Liu and Yinghao Chen and Yuxin Zuo and Cheng Qian and Kaiyan Zhang and Weize Chen and Chaojun Xiao and Ganqu Cui and Ning Ding and Zhiyuan Liu},
  year={2025},
  eprint={2512.16649},
  archivePrefix={arXiv},
  primaryClass={cs.CL},
}

@inproceedings{mao2026dynamicspredictive,
  title={{Dynamics-Predictive Sampling for Active RL Finetuning of Large Reasoning Models}},
  author={Yixiu Mao and Yun Qu and Cheems Wang and Heming Zou and Xiangyang Ji},
  booktitle={International Conference on Learning Representations (ICLR)},
  year={2026},
}

@inproceedings{sheng2024hybridflow,
  title={HybridFlow: A Flexible and Efficient RLHF Framework},
  url={http://dx.doi.org/10.1145/3689031.3696075},
  DOI={10.1145/3689031.3696075},
  booktitle={Proceedings of the Twentieth European Conference on Computer Systems},
  publisher={ACM},
  author={Sheng, Guangming and Zhang, Chi and Ye, Zilingfeng and Wu, Xibin and Zhang, Wang and Zhang, Ru and Peng, Yanghua and Lin, Haibin and Wu, Chuan},
  year={2025},
  month=Mar, pages={1279-1297},
}

@article{dong2023raft,
  title={{RAFT}: Reward rAnked FineTuning for Generative Foundation Model Alignment},
  author={Hanze Dong and Wei Xiong and Deepanshu Goyal and Yihan Zhang and Winnie Chow and Rui Pan and Shizhe Diao and Jipeng Zhang and KaShun SHUM and Tong Zhang},
  journal={Transactions on Machine Learning Research},
  issn={2835-8856},
  year={2023},
}

@article{singh2024beyond,
  title={Beyond Human Data: Scaling Self-Training for Problem-Solving with Language Models},
  author={Avi Singh and John D Co-Reyes and Rishabh Agarwal and Ankesh Anand and Piyush Patil and Xavier Garcia and Peter J Liu and James Harrison and Jaehoon Lee and Kelvin Xu and Aaron T Parisi and Abhishek Kumar and Alexander A Alemi and Alex Rizkowsky and Azade Nova and Ben Adlam and Bernd Bohnet and Gamaleldin Fathy Elsayed and Hanie Sedghi and Igor Mordatch and Isabelle Simpson and Izzeddin Gur and Jasper Snoek and Jeffrey Pennington and Jiri Hron and Kathleen Kenealy and Kevin Swersky and Kshiteej Mahajan and Laura A Culp and Lechao Xiao and Maxwell Bileschi and Noah Constant and Roman Novak and Rosanne Liu and Tris Warkentin and Yamini Bansal and Ethan Dyer and Behnam Neyshabur and Jascha Sohl-Dickstein and Noah Fiedel},
  journal={Transactions on Machine Learning Research},
  issn={2835-8856},
  year={2024},
}

@inproceedings{pang2024iterative,
  title = {{Iterative Reasoning Preference Optimization}},
  author = {Pang, Richard Yuanzhe and Yuan, Weizhe and Cho, Kyunghyun and He, He and Sukhbaatar, Sainbayar and Weston, Jason},
  booktitle = {Advances in Neural Information Processing Systems},
  doi = {10.52202/079017-3702},
  editor = {A. Globerson and L. Mackey and D. Belgrave and A. Fan and U. Paquet and J. Tomczak and C. Zhang},
  pages = {116617--116637},
  publisher = {Curran Associates, Inc.},
  volume = {37},
  year = {2024}
}

@inproceedings{wang2023selfconsistency,
  title={Self-Consistency Improves Chain of Thought Reasoning in Language Models},
  author={Xuezhi Wang and Jason Wei and Dale Schuurmans and Quoc V Le and Ed H. Chi and Sharan Narang and Aakanksha Chowdhery and Denny Zhou},
  booktitle={International Conference on Learning Representations (ICLR)},
  year={2023},
}

@misc{uesato2022solving,
  title={Solving math word problems with process- and outcome-based feedback},
  author={Jonathan Uesato and Nate Kushman and Ramana Kumar and Francis Song and Noah Siegel and Lisa Wang and Antonia Creswell and Geoffrey Irving and Irina Higgins},
  year={2022},
  eprint={2211.14275},
  archivePrefix={arXiv},
  primaryClass={cs.LG},
}

@inproceedings{lightman2024lets,
  title = {Let\textquotesingle s Verify Step by Step},
  author = {Lightman, Hunter and Kosaraju, Vineet and Burda, Yuri and Edwards, Harrison and Baker, Bowen and Lee, Teddy and Leike, Jan and Schulman, John  and Sutskever, Ilya and Cobbe, Karl},
  booktitle = {International Conference on Learning Representations (ICLR)},
  pages = {39578--39601},
  year = {2024}
}

@inproceedings{zhang2025generative,
  title = {Generative Verifiers: Reward Modeling as Next-Token Prediction},
  author = {Zhang, Lunjun and Hosseini, Arian and Bansal, Hritik and Kazemi, Seyed Mehran and Kumar, Aviral and Agarwal, Rishabh},
  booktitle = {International Conference on Learning Representations (ICLR)},
  pages = {12476--12505},
  year = {2025}
}

@misc{schulman2017proximal,
  title={Proximal Policy Optimization Algorithms},
  author={John Schulman and Filip Wolski and Prafulla Dhariwal and Alec Radford and Oleg Klimov},
  year={2017},
  eprint={1707.06347},
  archivePrefix={arXiv},
  primaryClass={cs.LG},
}

@article{williams1992simple,
  author = {Williams, Ronald J.},
  title = {Simple Statistical Gradient-Following Algorithms for Connectionist Reinforcement Learning},
  year = {1992},
  issue_date = {May 1992},
  volume = {8},
  number = {3-4},
  issn = {0885-6125},
  doi = {10.1007/BF00992696},
  journal = {Mach. Learn.},
  month = may,
  pages = {229-256},
  numpages = {28}
}

@inproceedings{wei2022chain,
  author = {Wei, Jason and Wang, Xuezhi and Schuurmans, Dale and Bosma, Maarten and Ichter, Brian and Xia, Fei and Chi, Ed and Le, Quoc V and Zhou, Denny},
  booktitle = {Advances in Neural Information Processing Systems},
  pages = {24824--24837},
  title = {Chain-of-Thought Prompting Elicits Reasoning in Large Language Models},
  volume = {35},
  year = {2022}
}

@inproceedings{yao2023tree,
  author = {Yao, Shunyu and Yu, Dian and Zhao, Jeffrey and Shafran, Izhak and Griffiths, Tom and Cao, Yuan and Narasimhan, Karthik},
  booktitle = {Advances in Neural Information Processing Systems},
  editor = {A. Oh and T. Naumann and A. Globerson and K. Saenko and M. Hardt and S. Levine},
  pages = {11809--11822},
  publisher = {Curran Associates, Inc.},
  title = {Tree of Thoughts: Deliberate Problem Solving with Large Language Models},
  volume = {36},
  year = {2023}
}

@inproceedings{zelikman2022star,
  author = {Zelikman, Eric and Wu, Yuhuai and Mu, Jesse and Goodman, Noah},
  booktitle = {Advances in Neural Information Processing Systems},
  editor = {S. Koyejo and S. Mohamed and A. Agarwal and D. Belgrave and K. Cho and A. Oh},
  pages = {15476--15488},
  publisher = {Curran Associates, Inc.},
  title = {STaR: Bootstrapping Reasoning With Reasoning},
  volume = {35},
  year = {2022}
}

@inproceedings{ouyang2022training,
  author = {Ouyang, Long and Wu, Jeffrey and Jiang, Xu and Almeida, Diogo and Wainwright, Carroll and Mishkin, Pamela and Zhang, Chong and Agarwal, Sandhini and Slama, Katarina and Ray, Alex and Schulman, John and Hilton, Jacob and Kelton, Fraser and Miller, Luke and Simens, Maddie and Askell, Amanda and Welinder, Peter and Christiano, Paul F and Leike, Jan and Lowe, Ryan},
  booktitle = {Advances in Neural Information Processing Systems},
  editor = {S. Koyejo and S. Mohamed and A. Agarwal and D. Belgrave and K. Cho and A. Oh},
  pages = {27730--27744},
  publisher = {Curran Associates, Inc.},
  title = {Training language models to follow instructions with human feedback},
  volume = {35},
  year = {2022}
}

@inproceedings{ahmadian2024back,
  title={Back to Basics: Revisiting {REINFORCE}-Style Optimization for Learning from Human Feedback in {LLM}s},
  author={Ahmadian Arash and Cremer Chris and Gall{\'e} Matthias and Fadaee Marzieh and Kreutzer Julia and Pietquin Olivier and {\"U}st{\"u}n Ahmet and Hooker Sara},
  booktitle={Proceedings of the 62nd Annual Meeting of the Association for Computational Linguistics (Volume 1: Long Papers)},
  month=aug,
  year={2024},
  publisher={Association for Computational Linguistics},
  doi={10.18653/v1/2024.acl-long.662},
  pages={12248--12267},
}

@misc{gulcehre2023rest,
  title={Reinforced Self-Training (ReST) for Language Modeling},
  author={Caglar Gulcehre and Tom Le Paine and Srivatsan Srinivasan and Ksenia Konyushkova and Lotte Weerts and Abhishek Sharma and Aditya Siddhant and Alex Ahern and Miaosen Wang and Chenjie Gu and Wolfgang Macherey and Arnaud Doucet and Orhan Firat and Nando de Freitas},
  year={2023},
  eprint={2308.08998},
  archivePrefix={arXiv},
  primaryClass={cs.CL},
}

@inproceedings{kool2019buy,
  author={Wouter Kool and Herke van Hoof and Max Welling},
  title={Buy 4 {REINFORCE} Samples, Get a Baseline for Free!},
  booktitle={Deep Reinforcement Learning Meets Structured Prediction, {ICLR} 2019 Workshop, New Orleans, Louisiana, United States, May 6, 2019},
  year={2019},
}

@misc{deepseek2024deepseekmath,
  title={DeepSeekMath: Pushing the Limits of Mathematical Reasoning in Open Language Models},
  author={Zhihong Shao and Peiyi Wang and Qihao Zhu and Runxin Xu and Junxiao Song and Xiao Bi and Haowei Zhang and Mingchuan Zhang and Y. K. Li and Y. Wu and Daya Guo},
  year={2024},
  eprint={2402.03300},
  archivePrefix={arXiv},
  primaryClass={cs.CL},
}

@article{deepseek2025deepseekr1,
  title={DeepSeek-R1 incentivizes reasoning in LLMs through reinforcement learning},
  volume={645},
  ISSN={1476-4687},
  DOI={10.1038/s41586-025-09422-z},
  number={8081},
  journal={Nature},
  author={Deepseek AI},
  year={2025},
  month=sep,
  pages={633--638}}

@misc{wei2025minimalist,
  title={A Minimalist Approach to {LLM} Reasoning: from Rejection Sampling to Reinforce},
  author={Wei Xiong and Jiarui Yao and Yuhui Xu and Bo Pang and Lei Wang and Doyen Sahoo and Junnan Li and Nan Jiang and Tong Zhang and Caiming Xiong and Hanze Dong},
  year={2025},
  eprint={2504.11343},
  archivePrefix={arXiv},
  primaryClass={cs.LG},
}

@inproceedings{bengio2009curriculum,
  author = {Bengio, Yoshua and Louradour, J\'{e}r\^{o}me and Collobert, Ronan and Weston, Jason},
  title = {Curriculum learning},
  year = {2009},
  isbn = {9781605585161},
  doi = {10.1145/1553374.1553380},
  booktitle = {Proceedings of the 26th Annual International Conference on Machine Learning},
  pages = {41--48},
  numpages = {8},
  series = {ICML '09}
}

@inproceedings{kumar2010self,
  title = {Self-Paced Learning for Latent Variable Models},
  author = {Kumar, M. and Packer, Benjamin and Koller, Daphne},
  booktitle = {Advances in Neural Information Processing Systems},
  pages = {1189--1197},
  volume = {23},
  year = {2010}
}

@article{abbas2023semdedup,
  title={SemDeDup: Data-Efficient Learning at Web-Scale through Semantic Deduplication},
  author={Abbas, Amro and Tirumala, Kushal and Simig, Daniel and Ganguli, Surya and Morcos, Ari S.},
  journal={arXiv preprint arXiv:2303.09540},
  eprint={2303.09540},
  archivePrefix={arXiv},
  year={2023}
}

@inproceedings{xia2024less,
  title={{LESS}: Selecting Influential Data for Targeted Instruction Tuning},
  author={Mengzhou Xia and Sadhika Malladi and Suchin Gururangan and Sanjeev Arora and Danqi Chen},
  booktitle={Forty-first International Conference on Machine Learning},
  series={Proceedings of Machine Learning Research},
  pages={54104--54132},
  year={2024}
}

@inproceedings{liu2024what,
  title = {What Makes Good Data for Alignment? A Comprehensive Study of Automatic Data Selection in Instruction Tuning},
  author = {Liu, Wei and Zeng, Weihao and He, Keqing and Jiang, Yong and He, Junxian},
  booktitle = {International Conference on Learning Representations (ICLR)},
  pages = {22353--22373},
  year = {2024}
}

@inproceedings{zhou2023lima,
  author = {Zhou, Chunting and Liu, Pengfei and Xu, Puxin and Iyer, Srinivasan and Sun, Jiao and Mao, Yuning and Ma, Xuezhe and Efrat, Avia and Yu, Ping and YU, LILI and Zhang, Susan and Ghosh, Gargi and Lewis, Mike and Zettlemoyer, Luke and Levy, Omer},
  booktitle = {Advances in Neural Information Processing Systems},
  pages = {55006--55021},
  title = {LIMA: Less Is More for Alignment},
  volume = {36},
  year = {2023}
}

@inproceedings{chen2024alpagasus,
  author = {Chen, Lichang and Li, Shiyang and Yan, Jun and Wang, Hai and Gunaratna, Kalpa and Yadav, Vikas and Tang, Zheng and Srinivasan, Vijay and Zhou, Tianyi and Huang, Heng and Jin, Hongxia},
  booktitle = {International Conference on Learning Representations (ICLR)},
  pages = {34767--34797},
  title = {AlpaGasus: Training a Better Alpaca with Fewer Data},
  year = {2024}
}

@inproceedings{shrivastava2016training,
  author = {Abhinav Shrivastava and Abhinav Gupta and Ross Girshick},
  title = {Training Region-based Object Detectors with Online Hard Example Mining},
  booktitle = {Proceedings of the IEEE Conference on Computer Vision and Pattern Recognition},
  year = {2016},
  month = {June},
  pages = {761--769},
}

@inproceedings{sorscher2022beyond,
  author = {Sorscher, Ben and Geirhos, Robert and Shekhar, Shashank and Ganguli, Surya and Morcos, Ari},
  booktitle = {Advances in Neural Information Processing Systems},
  pages = {19523--19536},
  title = {Beyond neural scaling laws: beating power law scaling via data pruning},
  volume = {35},
  year = {2022}
}

@inproceedings{brown2020language,
  author = {Brown, Tom and Mann, Benjamin and Ryder, Nick and Subbiah, Melanie and Kaplan, Jared D and Dhariwal, Prafulla and Neelakantan, Arvind and Shyam, Pranav and Sastry, Girish and Askell, Amanda and Agarwal, Sandhini and Herbert-Voss, Ariel and Krueger, Gretchen and Henighan, Tom and Child, Rewon and Ramesh, Aditya and Ziegler, Daniel and Wu, Jeffrey and Winter, Clemens and Hesse, Chris and Chen, Mark and Sigler, Eric and Litwin, Mateusz and Gray, Scott and Chess, Benjamin and Clark, Jack and Berner, Christopher and McCandlish, Sam and Radford, Alec and Sutskever, Ilya and Amodei, Dario},
  booktitle = {Advances in Neural Information Processing Systems},
  pages = {1877--1901},
  title = {Language Models are Few-Shot Learners},
  volume = {33},
  year = {2020}
}

@misc{cobbe2021training,
  title={Training Verifiers to Solve Math Word Problems},
  author={Karl Cobbe and Vineet Kosaraju and Mohammad Bavarian and Mark Chen and Heewoo Jun and Lukasz Kaiser and Matthias Plappert and Jerry Tworek and Jacob Hilton and Reiichiro Nakano and Christopher Hesse and John Schulman},
  year={2021},
  eprint={2110.14168},
  archivePrefix={arXiv},
  primaryClass={cs.LG},
}

@inproceedings{hendrycks2021math,
  author = {Hendrycks, Dan and Burns, Collin and Kadavath, Saurav and Arora, Akul and Basart, Steven and Tang, Eric and Song, Dawn and Steinhardt, Jacob},
  booktitle = {Proceedings of the Neural Information Processing Systems Track on Datasets and Benchmarks},
  title = {Measuring Mathematical Problem Solving With the {MATH} Dataset},
  volume = {1},
  year = {2021}
}

@inproceedings{kuhn2023semantic,
  title={Semantic Uncertainty: Linguistic Invariances for Uncertainty Estimation in Natural Language Generation},
  author={Lorenz Kuhn and Yarin Gal and Sebastian Farquhar},
  booktitle={International Conference on Learning Representations (ICLR)},
  year={2023},
}

@misc{liu2025understanding,
  title={Understanding R1-Zero-Like Training: A Critical Perspective},
  author={Zichen Liu and Changyu Chen and Wenjun Li and Penghui Qi and Tianyu Pang and Chao Du and Wee Sun Lee and Min Lin},
  year={2025},
  eprint={2503.20783},
  archivePrefix={arXiv},
  primaryClass={cs.LG},
}

@inproceedings{lin2017focal,
  author={Lin, Tsung-Yi and Goyal, Priya and Girshick, Ross and He, Kaiming and Dollar, Piotr},
  booktitle={2017 IEEE International Conference on Computer Vision (ICCV)},
  title={Focal Loss for Dense Object Detection},
  year={2017},
  pages={2999--3007},
  doi={10.1109/ICCV.2017.324}
}

@inproceedings{lin2024rho1,
  author = {Lin, Zhenghao and Gou, Zhibin and Gong, Yeyun and Liu, Xiao and Shen, Yelong and Xu, Ruochen and Lin, Chen and Yang, Yujiu and Jiao, Jian and Duan, Nan and Chen, Weizhu},
  booktitle = {Advances in Neural Information Processing Systems},
  doi = {10.52202/079017-0914},
  pages = {29029--29063},
  title = {Not All Tokens Are What You Need for Pretraining},
  volume = {37},
  year = {2024}
}

@inproceedings{lu2024instag,
  author = {Lu, Keming and Yuan, Hongyi and Yuan, Zheng and Lin, Runji and Lin, Junyang and Tan, Chuanqi and Zhou, Chang and Zhou, Jingren},
  booktitle = {International Conference on Learning Representations (ICLR)},
  pages = {36456--36474},
  title = {\#InsTag: Instruction Tagging for Analyzing Supervised Fine-tuning of Large Language Models},
  year = {2024}
}

@misc{openai2023gpt4,
  title={GPT-4 Technical Report},
  author={OpenAI},
  year={2023},
  eprint={2303.08774},
  archivePrefix={arXiv},
  primaryClass={cs.CL},
}

@inproceedings{pathak2017curiosity,
  author={Pathak, Deepak and Agrawal, Pulkit and Efros, Alexei A. and Darrell, Trevor},
  booktitle={2017 IEEE Conference on Computer Vision and Pattern Recognition Workshops (CVPRW)},
  title={Curiosity-Driven Exploration by Self-Supervised Prediction},
  year={2017},
  pages={488-489},
  doi={10.1109/CVPRW.2017.70}
}

@misc{qwen2025qwen3,
  title={Qwen3 Technical Report},
  author={An Yang and Anfeng Li and Baosong Yang and Beichen Zhang and Binyuan Hui and Bo Zheng and Bowen Yu and Chang Gao and Chengen Huang and Chenxu Lv and Chujie Zheng and Dayiheng Liu and Fan Zhou and Fei Huang and Feng Hu and Hao Ge and Haoran Wei and Huan Lin and Jialong Tang and Jian Yang and Jianhong Tu and Jianwei Zhang and Jianxin Yang and Jiaxi Yang and Jing Zhou and Jingren Zhou and Junyang Lin and Kai Dang and Keqin Bao and Kexin Yang and Le Yu and Lianghao Deng and Mei Li and Mingfeng Xue and Mingze Li and Pei Zhang and Peng Wang and Qin Zhu and Rui Men and Ruize Gao and Shixuan Liu and Shuang Luo and Tianhao Li and Tianyi Tang and Wenbiao Yin and Xingzhang Ren and Xinyu Wang and Xinyu Zhang and Xuancheng Ren and Yang Fan and Yang Su and Yichang Zhang and Yinger Zhang and Yu Wan and Yuqiong Liu and Zekun Wang and Zeyu Cui and Zhenru Zhang and Zhipeng Zhou and Zihan Qiu},
  year={2025},
  eprint={2505.09388},
  archivePrefix={arXiv},
  primaryClass={cs.CL},
}

@misc{qwen2024math,
  title={Qwen2.5-Math Technical Report: Toward Mathematical Expert Model via Self-Improvement},
  author={An Yang and Beichen Zhang and Binyuan Hui and Bofei Gao and Bowen Yu and Chengpeng Li and Dayiheng Liu and Jianhong Tu and Jingren Zhou and Junyang Lin and Keming Lu and Mingfeng Xue and Runji Lin and Tianyu Liu and Xingzhang Ren and Zhenru Zhang},
  year={2024},
  eprint={2409.12122},
  archivePrefix={arXiv},
  primaryClass={cs.CL},
}

@inproceedings{wang2023learning,
  title = {Learning from Mistakes via Cooperative Study Assistant for Large Language Models},
  author ={Wang, Danqing and Li, Lei},
  booktitle = {Proceedings of the 2023 Conference on Empirical Methods in Natural Language Processing},
  month = dec,
  year = {2023},
  doi = {10.18653/v1/2023.emnlp-main.659},
  pages = {10667--10685},
}

@inproceedings{xie2023data,
  author = {Xie, Sang Michael and Santurkar, Shibani and Ma, Tengyu and Liang, Percy S},
  booktitle = {Advances in Neural Information Processing Systems},
  pages = {34201--34227},
  title = {Data Selection for Language Models via Importance Resampling},
  volume = {36},
  year = {2023}
}

@inproceedings{yu2025dapo,
  author = {Yu, Qiying and Zhang, Zheng and Zhu, Ruofei and Yuan, Yufeng and Zuo, Xiaochen and Yue, Yu and Dai, Weinan and Fan, Tiantian and Liu, Gaohong and liu, juncai and Liu, LingJun and Liu, Xin and Lin, Haibin and Lin, Zhiqi and Ma, Bole and Sheng, Guangming and Tong, Yuxuan and Zhang, Chi and Zhang, Mofan and Zhang, Ru and Zhang, Wang and Zhu, Hang and Zhu, Jinhua and Chen, Jiaze and Chen, Jiangjie and Wang, Chengyi and Yu, Hongli and Song, Yuxuan and Wei, Xiangpeng and Zhou, Hao and Liu, Jingjing and Ma, Wei-Ying and Zhang, Ya-Qin and Yan, Lin and Wu, Yonghui and Wang, Mingxuan},
  booktitle = {Advances in Neural Information Processing Systems},
  pages = {113222--113244},
  title = {DAPO: An Open-Source LLM Reinforcement Learning System at Scale},
  volume = {38},
  year = {2025}
}

@misc{yuan2024self,
  title={Self-Rewarding Language Models},
  author={Weizhe Yuan and Richard Yuanzhe Pang and Kyunghyun Cho and Xian Li and Sainbayar Sukhbaatar and Jing Xu and Jason Weston},
  year={2024},
  eprint={2401.10020},
  archivePrefix={arXiv},
  primaryClass={cs.CL},
}

@misc{yuan2023scaling,
  title={Scaling Relationship on Learning Mathematical Reasoning with Large Language Models},
  author={Zheng Yuan and Hongyi Yuan and Chengpeng Li and Guanting Dong and Keming Lu and Chuanqi Tan and Chang Zhou and Jingren Zhou},
  year={2023},
  eprint={2308.01825},
  archivePrefix={arXiv},
  primaryClass={cs.CL},
}

@misc{qwen2024qwen25,
  title={{Qwen2.5} Technical Report},
  author={An Yang and Baosong Yang and Beichen Zhang and Binyuan Hui and Bo Zheng and Bowen Yu and Chengyuan Li and Dayiheng Liu and Fei Huang and Haoran Wei and Huan Lin and Jian Yang and Jianhong Tu and Jianwei Zhang and Jianxin Yang and Jiaxi Yang and Jingren Zhou and Junyang Lin and Kai Dang and Keming Lu and Keqin Bao and Kexin Yang and Le Yu and Mei Li and Mingfeng Xue and Pei Zhang and Qin Zhu and Rui Men and Runji Lin and Tianhao Li and Tianyi Tang and Tingyu Xia and Xingzhang Ren and Xuancheng Ren and Yang Fan and Yang Su and Yichang Zhang and Yu Wan and Yuqiong Liu and Zeyu Cui and Zhenru Zhang and Zihan Qiu},
  year={2024},
  eprint={2412.15115},
  archivePrefix={arXiv},
  primaryClass={cs.CL},
}

\newpage
\appendix
\addtocontents{toc}{\protect\setcounter{tocdepth}{2}}
\renewcommand{\contentsname}{Appendices}
\tableofcontents
\newpage
\section{Case Study: Data Selection Rationale}
\label{sec:case_study}
\begin{figure}[H]
\centering
\includegraphics[width=\textwidth]{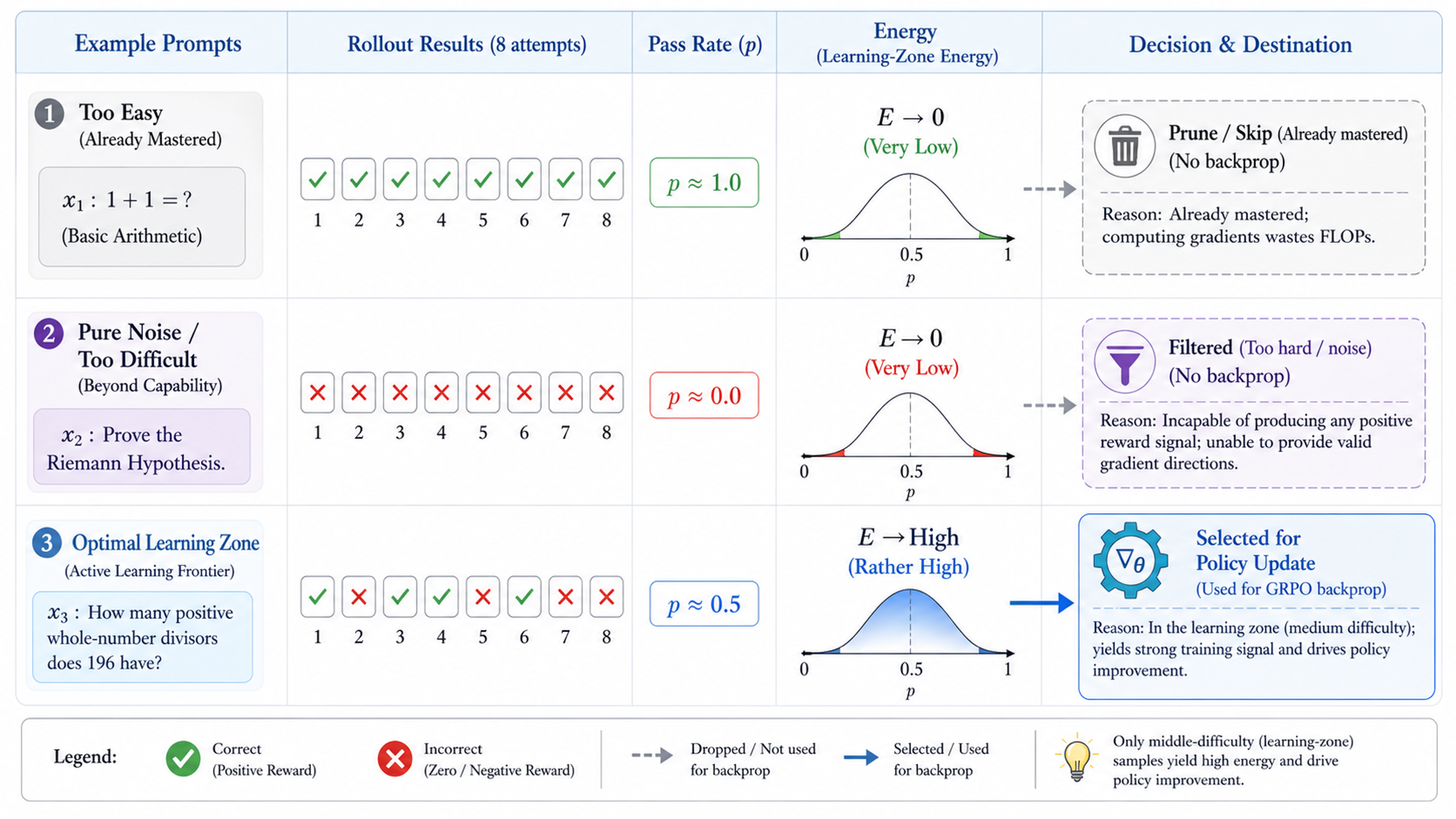}
\caption{\textbf{Conceptual case study of the learning zone.} Trivial prompts that are already solved ($p \approx 1$) and overwhelmingly hard prompts that consistently fail ($p \approx 0$) carry negligible Learning-Zone Energy and are therefore deprioritized. In contrast, frontier prompts with mixed rollout outcomes ($p \approx 0.5$) receive the highest uncertainty weight and remain actively selected for GRPO updates.}
\label{fig:case_study}
\end{figure}

\section{Details of Models and Datasets}
\label{sec:model_data_details}

\paragraph{Models.}
We evaluate two families of open-weight Qwen reasoning models. The Qwen2.5-Math line, including Qwen2.5-Math-1.5B-Instruct (\url{https://huggingface.co/Qwen/Qwen2.5-Math-1.5B-Instruct}) and Qwen2.5-Math-7B-Instruct (\url{https://huggingface.co/Qwen/Qwen2.5-Math-7B-Instruct}), is specialized for mathematical reasoning and provides our main math-centric backbone. The Qwen3 line, including Qwen3-1.7B (\url{https://huggingface.co/Qwen/Qwen3-1.7B}), Qwen3-4B (\url{https://huggingface.co/Qwen/Qwen3-4B}), and Qwen3-8B (\url{https://huggingface.co/Qwen/Qwen3-8B}), offers broader reasoning capacity and stronger general-purpose post-training behavior; for all Qwen3 experiments we enable thinking mode throughout training and evaluation. According to the corresponding Hugging Face model cards, all five checkpoints are released under the Apache-2.0 license.

\paragraph{Datasets.}
We train on three public mathematical reasoning datasets with progressively increasing difficulty. GSM8K (\url{https://huggingface.co/datasets/openai/gsm8k}) contains relatively short grade-school arithmetic word problems and serves as the easiest setting in our study. MATH (\url{https://huggingface.co/datasets/EleutherAI/hendrycks_math}) contains substantially harder competition-style problems spanning algebra, geometry, counting, and number theory. DAPO-MATH (English split; \url{https://huggingface.co/datasets/BytedTsinghua-SIA/DAPO-Math-17k}) is the most challenging training source used here, with longer solution trajectories and more diverse reasoning patterns, making it especially suitable for probing efficiency gains under harder post-training regimes. For evaluation, we use the corresponding held-out test splits of GSM8K, MATH, and DAPO-MATH from the same repositories, together with the OOD benchmarks AMC23 (\url{https://huggingface.co/datasets/JingzeShi/amc23}) and AIME25 (\url{https://huggingface.co/datasets/math-ai/aime25}); these are exactly the benchmark sets reported in Table~\ref{tab:pass1_all} and Table~\ref{tab:comparison_methods} of Section~\ref{sec:experiments}.

\paragraph{Licenses.}
The external assets used in our experiments are all taken from public model or dataset cards that state their licenses explicitly. The Qwen checkpoints are Apache-2.0; GSM8K and MATH are MIT; DAPO-MATH and AIME25 are Apache-2.0; and the AMC23 mirror used in our evaluation setup is MIT.

\section{Chat Templates and Prompt Formatting}
\label{sec:chat_template}

\begin{tcolorbox}[colback=green!5!white,colframe=green!50!black,title=ChatML System and User Template]
\begin{verbatim}
<|im_start|>system
You are a helpful mathematical reasoning assistant.
<|im_end|>
<|im_start|>user
Please reason step by step, and
put your final answer within \boxed{}.
{Question}
<|im_end|>
<|im_start|>assistant
\end{verbatim}
\end{tcolorbox}

\section{Hardware and Runtime Details}
\label{sec:hardware_runtime}

All experiments are conducted on 8 NVIDIA Blackwell-series GPUs. For runs trained on GSM8K or MATH, each configuration requires roughly 3.5 hours on average. Runs trained on DAPO-MATH are longer because of the more difficult prompts and longer reasoning trajectories, with an average duration of about 8 hours per configuration. Relative to the full-data baseline, our method reduces end-to-end runtime by approximately 30\%--40\% across the reported settings.

\section{Detailed Computational Efficiency Analysis}
\label{sec:flops_derivation}

In this section, we provide a rigorous mathematical derivation of the computational footprint associated with our reinforcement learning (RL) training pipeline. We contextualize our footprint within the standard \texttt{veRL} framework execution to derive the baseline theoretical coefficient of $10$, which we partition into inference and optimization phases, measured in floating-point operations (FLOPs) per parameter per token.

Let $P$ denote the number of parameters in the policy model, and $T$ represent the average total sequence length ($T = L_{\mathrm{prompt}} + L_{\mathrm{response}}$).

\textbf{Inference and Generation Cost.}
During the data collection phase, the policy model autoregressively generates $n_{\mathrm{rollout}}$ responses for each prompt. Generating a single token requires a forward pass through the active policy network. In dense causual Transformer architectures, KV-cache-assisted token decoding incurs approximately $2P$ FLOPs per token.
Following standard engineering implementations of GRPO (e.g., the \texttt{veRL} framework using vLLM), the evaluation for the reference model logic is computationally merged or evaluated without gradient tracking. A reward model is subsequently invoked to evaluate the sequences, costing an additional $2P$ FLOPs per token for the forward pass.
Therefore, the generative inference scaling factor analytically resolves to:
Eq.~\eqref{eq:flops_infer}:
\begin{equation}
\label{eq:flops_infer}
C_{\mathrm{infer}} = \underbrace{2}_{\text{Actor Generation}} + \underbrace{2}_{\text{Reward Model Scoring}} = 4.
\end{equation}

\textbf{Optimization and Update Cost.}
During the optimization phase, only the actor (policy) model's weights undergo iterative updates via backpropagation. No critic model is utilized in GRPO. For a decoder-only LLM, the optimization phase incorporates a forward pass (activations) and a backward pass (gradients with respect to activations and learnable weights). The forward pass necessitates $2P$ FLOPs per token. The backward pass demands approximately $4P$ FLOPs per token ($2P$ for activation gradients, and $2P$ for weight gradients).
The underlying computational coefficient for optimization is modeled strictly as:
Eq.~\eqref{eq:flops_optim}:
\begin{equation}
\label{eq:flops_optim}
C_{\mathrm{optim}} = \underbrace{2}_{\text{Training Forward}} + \underbrace{4}_{\text{Training Backward}} = 6.
\end{equation}

\textbf{Total Baseline and Data Selection Savings.}
Cumulatively, the theoretical cost baseline per token scales with $C_{\mathrm{total}} = C_{\mathrm{infer}} + C_{\mathrm{optim}} = 10$. For a global dataset $D$ and epoch count $E$, the baseline expense is given by Eq.~\eqref{eq:flops_base}:
\begin{equation}
\label{eq:flops_base}
\mathrm{FLOPs}_{\mathrm{base}} \approx 10 \cdot E \cdot |D| \cdot n_{\mathrm{rollout}} \cdot P \cdot T.
\end{equation}

Our learning-zone data filtering regimen dynamically selects and retains only the prime $40\%$ of sample groups mapping backpropagation. Generative inferences and objective energy scorings remain necessarily fully evaluated. The modified compute footprint explicitly rescales the optimization term, yielding Eq.~\eqref{eq:flops_ours}:
\begin{equation}
\label{eq:flops_ours}
\mathrm{FLOPs}_{\mathrm{ours}} \approx (C_{\mathrm{infer}} + 0.4 \cdot C_{\mathrm{optim}}) \cdot E \cdot |D| \cdot n_{\mathrm{rollout}} \cdot P \cdot T = (4 + 0.4 \cdot 6) \cdot (\dots) = 6.4 \cdot (\dots).
\end{equation}
Consequently, the structural computation conserved corresponds firmly to Eq.~\eqref{eq:flops_saved}:
\begin{equation}
\label{eq:flops_saved}
\mathrm{Saved\ FLOPs} = 1 - \frac{6.4}{10} = 36\%.
\end{equation}
This exact reduction isolates and eliminates redundant backward passes across uninformative state manifolds (fully-saturated or intrinsically out-of-distribution regimes), drastically widening practical scaling capabilities without compromising policy alignment limits.

\section{Additional Ablation Studies}
\label{sec:appendix_ablation}

\textbf{Impact of Rollout Count $N$: Unconstrained Upper Bound.}
Section~\ref{sec:ablation} (Figure~\ref{fig:ablation_rollout_n}) evaluates the effect of rollout count under a \emph{fixed total rollout budget} of $7.8\times10^5$. Here we remove the budget constraint and instead run each $N$ configuration for a fixed number of gradient steps (equal training iterations), so each experiment receives a different but unconstrained compute allocation. This measures the quality upper bound achievable by LZE at each $N$ independently of cost.

\begin{figure}[h]
  \centering
  \includegraphics[width=0.52\linewidth]{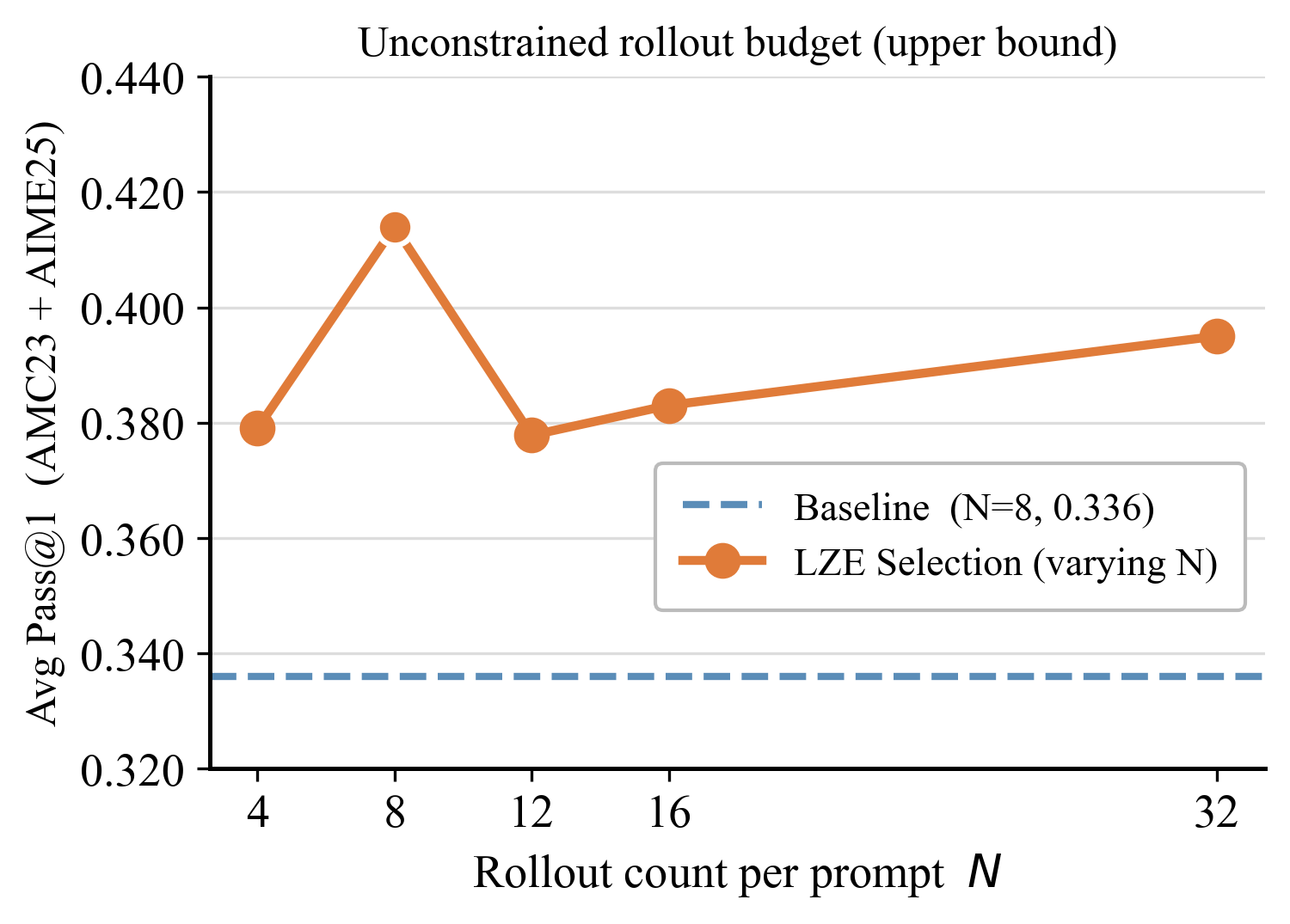}
  \caption{Rollout-$N$ sensitivity under an \emph{unconstrained} budget. Each $N$ runs for the same number of gradient steps; compute cost scales with $N$. Even under unconstrained conditions, LZE selection consistently outperforms the baseline at every $N$, and $N{=}8$ again achieves the best absolute performance.}
  \label{fig:ablation_rollout_n_appendix}
\end{figure}

As shown in Figure~\ref{fig:ablation_rollout_n_appendix}, the qualitative pattern mirrors the fixed-budget result: LZE selection outperforms the full-data baseline at every rollout count, and performance peaks at $N{=}8$ (Avg Pass@1 $= 0.414$ vs.\ baseline $0.336$).
Smaller $N$ ($= 4$) yields noisier group-level pass-rate estimates, making the Energy Score less reliable; larger $N$ ($\geq 12$) provides more stable estimates but offers diminishing returns—the gain from additional rollouts within a single step is small compared to the benefit of more gradient steps using the same compute. Together, the fixed-budget and unconstrained results confirm that $N{=}8$ is a robust default across both efficiency-constrained and quality-maximising regimes.

\textbf{Sensitivity to the Momentum Coefficient $\alpha$.}
The momentum term $(1+\alpha m_i^{(t)})$ in the Energy Score amplifies prompts where the pass rate is actively improving, with $\alpha = 0$ reducing the score to a purely uncertainty-based criterion (no momentum). We sweep $\alpha \in \{0,\,0.15,\,0.3,\,0.45\}$ on MATH with Qwen2.5-Math-1.5B and Qwen2.5-Math-7B, reporting Avg Pass@1 = (AMC23 + AIME25) / 2.

\begin{figure}[h]
  \centering
  \includegraphics[width=0.52\linewidth]{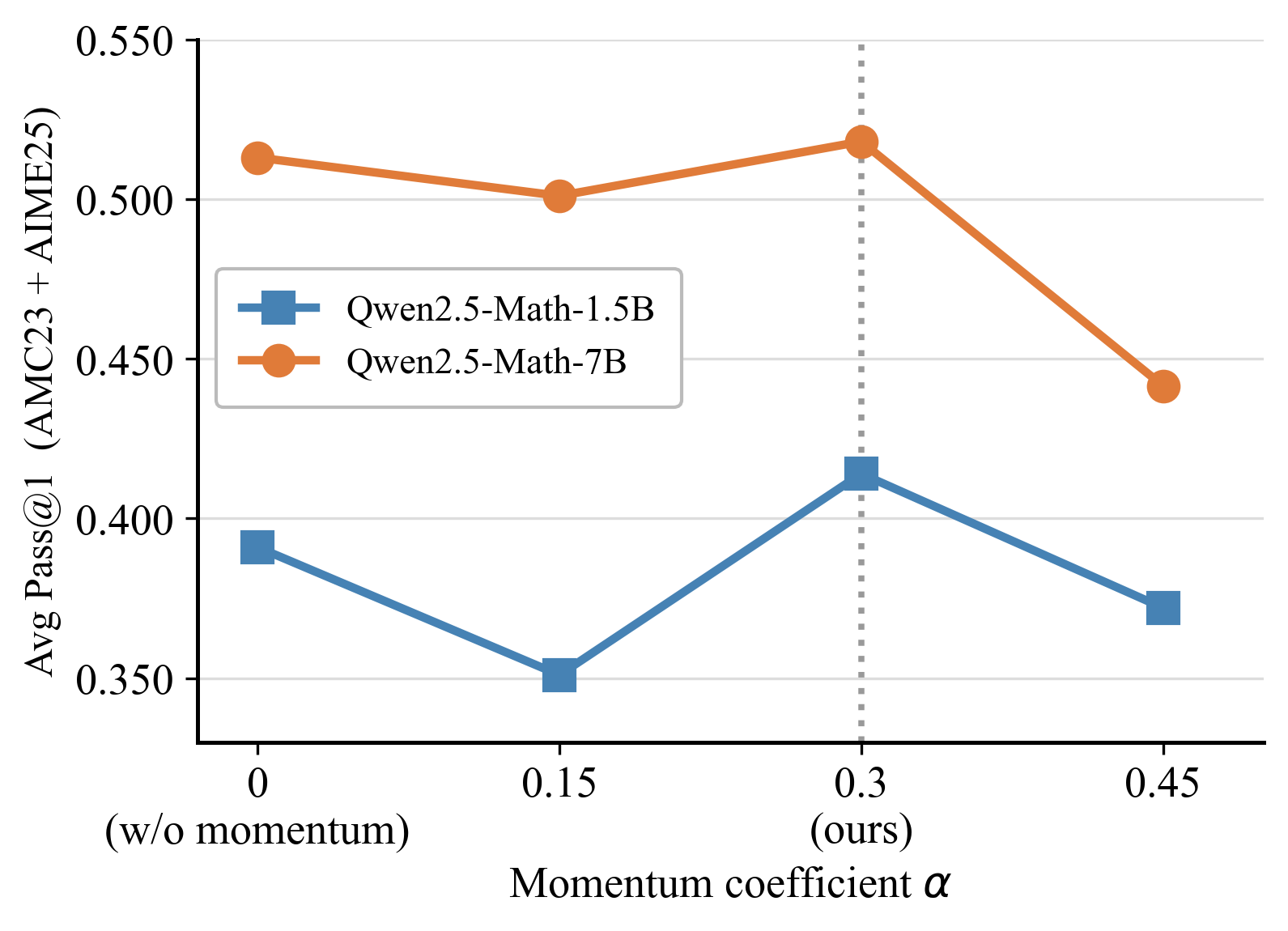}
  \caption{Sensitivity to the momentum coefficient $\alpha$. Both 1.5B and 7B models show consistent peaks at $\alpha = 0.3$ (our default). At $\alpha = 0$ the momentum term is absent; at $\alpha = 0.45$ the score becomes overly reactive, destabilising selection.}
  \label{fig:ablation_alpha}
\end{figure}

Figure~\ref{fig:ablation_alpha} reveals a consistent non-monotone pattern across both model scales. Setting $\alpha = 0$ (no momentum) already yields strong performance—confirming that the uncertainty term is the primary driver—but the momentum signal provides a meaningful further boost at $\alpha = 0.3$: +0.023 for 1.5B (0.391$\to$0.414) and +0.005 for 7B (0.513$\to$0.518). A smaller $\alpha = 0.15$ under-utilises the temporal change signal and falls below the no-momentum baseline for the 1.5B model, suggesting that weak momentum can transiently mislead selection before the EMA has time to smooth out. An excessively large $\alpha = 0.45$ makes the score overly sensitive to single-step fluctuations in the pass rate: transient improvements are over-rewarded, causing the selector to repeatedly re-select prompts that exhibited one lucky run but have not yet stably improved, which degrades the training distribution and reduces final accuracy. The sweet spot at $\alpha = 0.3$ strikes the best balance between responsiveness and stability across both model sizes, and we adopt it as the default for all experiments in the main paper.

\section{Use of LLMs}
The drafting of this manuscript was enhanced through the use of a large language model, assisting in grammatical refinement.

\section{Broader Impacts}
The potential negative social impacts of our method are minor, and align with those typically associated with general LLM reasoning technologies. We emphasise the importance of adhering to the principles of fair and safe deployment of LLM systems, including careful consideration of the training data, transparency about model capabilities and limitations, and ongoing monitoring for unintended consequences.
\end{document}